\documentclass{article}
\usepackage{arxiv}
\usepackage[utf8]{inputenc} 
\usepackage[T1]{fontenc}    
\usepackage{hyperref}
\usepackage{url}            
\usepackage{booktabs}       
\usepackage{amsfonts}       
\usepackage{nicefrac}       
\usepackage{microtype}      
\usepackage{amsmath}
\usepackage{amssymb}
\usepackage{graphicx}
\usepackage{subfigure}
\usepackage[round]{natbib}
\renewcommand{\cite}[1]{\citep{#1}}
\author{Takeshi Teshima,\textsuperscript{\rm{1,2}}
Issei Sato,\textsuperscript{\rm{1,2}}
Masashi Sugiyama\textsuperscript{\rm{2,1}}\\
\textsuperscript{\rm{1}}The University of Tokyo, \textsuperscript{\rm{2}}RIKEN \\
teshima@ms.k.u-tokyo.ac.jp,
\{sato, sugi\}@k.u-tokyo.ac.jp}
\newcommand{\todo}[1]{}
\newcommand{\todotwo}[1]{}
\newcommand{\todothree}[1]{}
\def \figureRoot {.}

\def \ptr {p_\rmtar}
\def \ptar {\ptr}
\def \psrc {p_\rmsrc}
\def \psrck {p_k}
\def \qsrck {q_k}
\def \qtar {q_\rmtar}
\def \Etr {\E_{\ptar}}
\def \Zsp {\mathcal{Z}}
\def \Xsp {\mathcal{X}}
\def \Ysp {\mathcal{Y}}
\def \Z {Z}
\def \z {z}
\def \DataTar {\mathcal{D}_\rmtar}
\def \Datak {\Datakk{k}}
\newcommand{\Datakk}[1]{\mathcal{D}_{#1}}

\makeatletter
\newcommand{\subalign}[1]{%
\vcenter{%
\Let@ \restore@math@cr \default@tag
\baselineskip\fontdimen10 \scriptfont\tw@
\advance\baselineskip\fontdimen12 \scriptfont\tw@
\lineskip\thr@@\fontdimen8 \scriptfont\thr@@
\lineskiplimit\lineskip
\ialign{\hfil$\m@th\scriptstyle##$&$\m@th\scriptstyle{}##$\hfil\crcr
#1\crcr
}%
}%
}
\makeatother
\def \Zk {{Z^{\rmsrc}_k}}
\def \Zki {{Z^{\rmsrc}_{k, i}}}
\def \Zi {{\Z_i}}

\def \Ski {\S^{\rmsrc}_{k,i}}
\def \Si {\S_i}
\def \augSi {\bar \s_{\augi}}
\def \augZi {\bar \z_{\augi}}
\def \augi {{\bm{i}}}

\def \Qsp {\mathcal{Q}}
\DeclareMathOperator{\NLICA}{ICA}
\def \argmin {\mathop{\rm arg~min}\limits}
\def \argming {\argmin_{\g \in \G}}
\DeclareMathOperator{\AllCombinations}{AllCombinations}

\def \NNPsid {\Psi} 
\def \rhf {r_{\hf, \bm{\psi}}}

\usepackage{amsmath}
\usepackage{amssymb}
\usepackage{dsfont}
\usepackage{bbm}
\usepackage{amsmath}
\def \argmin {\mathop{\rm argmin}\limits}
\newcommand{\comb}[2]{{#1 \choose #2}}

\def \iid {\overset{\text{i.i.d.}}{\sim}}

\def \Var {\mathrm{Var}}

\def \u {\bf{u}}
\def \x {\bf{x}}
\def \E {{\mathbb{E}}}
\def \Order {\mathcal{O}}

\def \prodd {\prod_d}

\def \Unif {\mathrm{U}}

\def \USsp {\mathcal{\tilde{S}}} 

\def \Xsp {\mathcal{X}}
\def \Ysp {\mathcal{Y}}

\def \u {u}
\def \Usp {\mathcal{U}}
\def \z {z}
\def \Zsp {\mathcal{Z}}

\def \rmsrc {\mathrm{Src}}
\def \rmtar {\mathrm{Tar}}
\def \nk {n_k}
\def \nsrck {\nk}
\def \ntr {n_\rmtar}

\def \Holder {H{\"o}lder\ }
\providecommand{\annot}[2]{\underbrace{#1}_{\text{#2}}}
\providecommand{\overbar}[1]{\mkern 1.5mu\overline{\mkern-1.5mu#1\mkern-1.5mu}\mkern 1.5mu}
\providecommand{\Mean}[1]{\mathop{\overbar{\sum}}\limits_{#1}}

\DeclareMathOperator{\Support}{supp}
\def \Z {Z}
\def \S {S}

\def \d {\mathrm{d}}
\def \de {\d\e}
\def \dx {\d\x}
\def \ds {\d\s}
\def \ED {\cQ^D}
\newcommand{\ETwoToD}[1]{\E_{#1_2, \ldots, #1_D}}

\newcommand{\EOneToD}[1]{\E_{#1_1, \ldots, #1_D}}
\def \Re {\mathbb{R}}
\def \RePos {\Re_{\geq 0}}
\def \ReD {\Re^D}
\def \ReDminus {\Re^{D-1}}
\def \Na {\mathbb{N}}
\def \Zsp {\mathcal{Z}}
\def \G {\mathcal{G}}
\def \F {\mathcal{F}}

\def \PDsp {\mathcal{P}}

\def \Qsp {\mathcal{Q}}

\newcommand{\HolderSp}[2]{C^{#1, #2}}

\def \UEspVol {V_\USsp}
\def \CubeLipschitz {\mathcal{H}}

\newcommand{\qdd}[1]{\q^{(#1)}}
\def \qd {\qdd{d}}
\def \cqd {\cq^{(d)}}

\newcommand{\pushForward}[2]{#1_\sharp #2}

\def \Jacobian {\mathrm{J}}
\newcommand{\J}[2]{\left|(\Jacobian#1)(#2)\right|}
\newcommand{\Ji}[2]{(\Jacobian#1)(#2)}

\def \supg {\sup_{\g \in \G}}
\def \infg {\inf_{\g \in \G}}

\def \allg {\forall \g \in \G}
\def \allq {\forall \q \in \Qsp}

\def \supeFull {\sup_{\e \in \ReD}}

\def \supzFull {\sup_{\z \in \ReD}}
\def \supDiffZ {\sup_{\z_1 \neq \z_2}}
\newcommand{\meanX}[1]{\frac{1}{n}\sum_{#1=1}^n}
\def \Erad {\E_\rad}
\def \hEData {\hat \E_n}

\newcommand{\deriv}[2]{\frac{\mathrm{d}#1}{\mathrm{d}#2}}
\newcommand{\pderiv}[2]{\frac{\partial #1}{\partial #2}}

\def \maxk {\max_{k \in [D]}}
\def \sumj {\sum_{j=1}^D}
\def \sumk {\sum_{k=1}^D}
\newcommand{\DefNorm}[2]{\left\|#1\right\|_{#2}}
\def \LoneSp {{L^1}}

\def \LdSp {{L^d}}
\def \LDSp {{L^D}}
\def \LpSp {{L^p}}

\def \ltwoSp {{\ell^2}}
\def \loneSp {{\ell^1}}
\newcommand{\Lone}[1]{\DefNorm{#1}{\LoneSp}}

\newcommand{\Lp}[1]{\DefNorm{#1}{\LpSp}}

\newcommand{\LoneWith}[2]{\DefNorm{#2}{\LoneSp(#1)}}

\newcommand{\LdWith}[2]{\DefNorm{#2}{\LdSp(#1)}}
\newcommand{\LDWith}[2]{\DefNorm{#2}{\LDSp(#1)}}

\newcommand{\finvHolderNorm}[1]{\DefNorm{#1}{\finvHolderClass}}

\newcommand{\SobolevSp}[2]{W^{#1,#2}}
\newcommand{\Sobolev}[3]{\DefNorm{#3}{\SobolevSp{#1}{#2}}}
\newcommand{\SobolevOne}[1]{\Sobolev{1}{1}{#1}}
\newcommand{\Sobolevd}[1]{\Sobolev{1}{d}{#1}}
\newcommand{\SobolevD}[1]{\Sobolev{1}{D}{#1}}
\newcommand{\op}[1]{\DefNorm{#1}{\mathrm{op}}}

\newcommand{\supnrm}[1]{\DefNorm{#1}{\infty}}
\newcommand{\ltwo}[1]{\DefNorm{#1}{\ltwoSp}}
\newcommand{\lone}[1]{\DefNorm{#1}{\loneSp}}
\def \msupnrm {\supnrm}
\newcommand{\opone}[1]{\DefNorm{#1}{\mathrm{op}(1)}}

\def \R {R}
\def \hR {\hat{\R}}
\def \cbarR {\bar{\check{\R}}}
\def \cbarhR {\check{\R}}

\def \LossValSp {[0, \LossUpperBound]}
\def \l {\ell}
\def \Data {\mathcal{D}}
\def \cData {\check{\mathcal{D}}}

\def \cEData {\E_{\cData}}
\def \iid {\overset{\text{i.i.d.}}{\sim}}

\newcommand{\TakeProductMeasure}[2]{#1^{#2}}
\def \q {q}
\def \Q {Q}
\def \Qn {\TakeProductMeasure{Q}{n}}
\def \Qj {\TakeProductMeasure{Q}{j}}

\def \cq {\check{\q}}
\def \cQ {\check{\Q}}
\def \cQn {\TakeProductMeasure{\cQ}{n}}
\def \cQj {\TakeProductMeasure{\cQ}{j}}

\newcommand{\Probability}[1]{\mathbb{P}_{#1}}
\def \p {p}

\def \x {z}
\def \e {s}
\def \s {s}

\def \g {g}
\def \gstar {g^*}
\def \cbarg {\bar{\check{\g}}}
\def \cbarhg {\check{\g}}
\def \hg {\hat\g}

\def \th {\theta}

\def \jGroupSplit {\mathfrak{G}_j^D}
\def \jGroupSplitDef {\{\tau : [D] \to [j] \ \vert\ \tau \text{ is surjective}\}}

\def \jGroupSplitCard {|\jGroupSplit|}
\def \jGroupSplitMean {\Mean{\tau \in \jGroupSplit}}
\def \DSymmetric {\mathfrak{S}_D}
\def \DSymmetricMean {\Mean{\pi \in \DSymmetric}}
\newcommand{\nCombset}[1]{\mathfrak{I}_#1^n}
\def \jnCombset {\nCombset{j}}
\def \jnCombsetDef {\{\rho : [j] \to [n] \ \vert\ \rho \text{ is injective}\}}

\def \jnCombsetMean {\Mean{\rho \in \jnCombset}}

\newcommand{\Vn}[1]{V_n^{#1}}
\newcommand{\Un}[1]{U_n^{#1}}
\def \Unj {\Un{j}}
\def \GUn {\mathrm{GU}_{(n_1, \ldots, n_L)}^{(k_1, \ldots, k_L)}}
\def \Rademacher {\mathfrak{R}}
\def \Rad {\Rademacher}
\def \RadG {\Rad(\G)}
\def \ORad {\Rad_\mathrm{ord}}

\newcommand{\metricEntropy}[3]{\mathcal{N}(#1, #2, #3)}
\def \f {f}
\def \finv {f^{-1}}
\def \hf {\hat f}
\def \hfinv {{\hf}^{-1}}

\def \dhf {\delta\hf}

\def \finvhf {\finv\circ\hf}

\def \finvf {\finv\circ\f}
\def \dfinvf {\deriv{\finvf}}
\def \dfinvhf {\deriv{\finvhf}}
\def \dfinv {\deriv{\finv}}

\def \pfinvj {\pderiv{\finv_j}}
\def \df {\deriv{\f}}

\def \pdfj {\pderiv{\f_j}}
\def \dhf {\deriv{\hf}}

\def \pdhfj {\pderiv{\hf_j}}
\newcommand{\Uke}[1]{\ke^{(#1)}}
\newcommand{\dUke}[1]{\ke_{#1}}

\def \ke {\tilde \l}
\def \keSp {\Phi}

\def \Ukej {\Uke{j}}
\def \lstar {\l^\dagger}
\def \lstarj {\l^{\dagger(j)}}
\newcommand{\UkeSp}[1]{\Phi_{\G, \hf}^{(#1)}}
\def \leadingKe {\ke^*}
\def \nsrc {n_\mathrm{src}}

\def \N {N}

\def \rad {\sigma}

\def \LossUpperBound {B_{\l}}
\def \qUpperBound {B_{\q}}
\def \qLipschitzConst {L_{\q}}
\def \finvLipschitzConst {L_{\finv}}
\def \lossUniformLipschitzConst {L_{\l_\G}}
\def \finvHolderClass {\HolderSp{1}{1}}
\def \dfSupNormConst {B_{\partial\f}^{\infty}}
\def \dfinvSupNormConst {B_{\partial\finv}^{\infty}}

\def \finalErrorBound {\annot{\ApproxErrorUpperBoundLeading}{Approximation error} + \annot{\PseudoGeneralizationErrorBoundForFinalErrorBoundOne}{Estimation error} + \annot{\finalErrorBoundSecondHigherOrder + \ApproxErrorUpperBoundResidual}{Higher order terms}}

\def \finalErrorBoundSecondHigherOrder {\kappa_1(\delta', n)}
\def \finalErrorBoundSecondHigherOrderDef {\Order(n^{-1})/{\delta'} + \Order(n^{-1})}
\def \ApproxErrorUpperBound {\ApproxErrorUpperBoundLeading + \ApproxErrorUpperBoundResidual}
\def \ApproxErrorUpperBoundLeading {\ApproxErrorUpperBoundConstOne \sum_{j=1}^D \SobolevOne{\f_j - \hf_j}}
\def \ApproxErrorUpperBoundResidual {D \LossUpperBound \qUpperBound \ApproxDensityDifferenceBoundRemainder}
\def \ApproxErrorUpperBoundContent {(\ApproxErrorUpperBoundConstOneDef) \sum_{j=1}^D \SobolevOne{\f_j - \hf_j} + D \LossUpperBound \qUpperBound \ApproxDensityDifferenceBoundRemainder}
\def \ApproxErrorUpperBoundConstOne {C}
\def \ApproxErrorUpperBoundConstOneDef {\qUpperBound \lossUniformLipschitzConst + D \LossUpperBound (\ApproxDensityDifferenceBoundMainCoeff)}
\def \ApproxDensityDifferenceBoundMain {(\ApproxDensityDifferenceBoundMainCoeff)\sum_{j=1}^D \SobolevOne{\f_j - \hf_j}}
\def \ApproxDensityDifferenceBoundMainCoeff {\qLipschitzConst \finvLipschitzConst + \qUpperBound D \ApproxDensityDifferenceJacobianPieceCoeffOne}
\def \ApproxDensityBoundOne {\qLipschitzConst \finvLipschitzConst \sum_{j=1}^D \SobolevOne{\f_j - \hf_j}}
\def \ApproxDensityDifferenceBoundRemainder {\kappa_2(\f - \hf)}
\def \ApproxDensityDifferenceBoundRemainderDef {\sum_{d=2}^D \comb{D}{d} \ApproxDensityDifferenceJacobianPieceCoeff \sum_{j=1}^D \Sobolevd{\f_j - \hf_j}^d}
\def \ApproxDensityBoundTwo {D \ApproxDensityDifferenceJacobianPieceCoeffOne \sum_{j=1}^D \SobolevOne{\f_j - \hf_j}}

\def \ApproxLossDifferenceBound {\qUpperBound \lossUniformLipschitzConst \sum_{j=1}^D \SobolevOne{\f_j - \hf_j}}
\def \ApproxDensityDifferenceJacobianPieceCoeff {C_d'}
\def \ApproxDensityDifferenceJacobianPieceCoeffOne {C_1'}
\def \ApproxDensityDifferenceJacobianPieceCoeffDef {(D+1)^{\frac{7}{2}d-2} \left((\dfSupNormConst)^d\left(\sumk \finvHolderNorm{\finv_k}\right)^d + (\dfinvSupNormConst)^d\right)}
\def \ApproxDensityDifferenceJacobianPieceCoeffOneDef {(D+1)^{3/2} \left(\dfSupNormConst\left(\sumk \finvHolderNorm{\finv_k}\right) + \dfinvSupNormConst\right)}

\def \PseudoGeneralizationErrorBoundOne {4D\wD\RadG + 2D\LossUpperBound \wD\PseudoGenProbTerm{\delta}}
\def \PseudoGeneralizationErrorBoundForFinalErrorBoundOne {4D\RadG + 2D\LossUpperBound \PseudoGenProbTerm{\delta}}
\def \PseudoGeneralizationErrorBoundTwo {2\wD(D-1)\sum_{j=2}^D\frac{\PseudoGenCj}{\delta'} n^{-j/2}+ 4 \LossUpperBound \sum_{j=1}^{D-1} w_j}
\def \PseudoGenCj {C_j}
\def \PseudoGenWj {w_j}
\def \wD {w_D}
\def \wDContent {\frac{n(n-1) \cdots (n-D+1)}{n^D}}
\newcommand{\PseudoGenProbTerm}[1]{\sqrt{\frac{\log 2 / {#1}}{2 n}}}

\newcommand{\tips}[1]{}
\def \IIDtext {i.i.d.}
\def \IID {\IIDtext\ }
\def \Supplementary {Supplementary~Material}
\newcommand{\unitnum}[2]{#1\;#2}

\usepackage{algorithm}
\usepackage{algorithmic}
\usepackage{tikz}
\usetikzlibrary{positioning}
\usepackage[font=small,labelfont=bf]{caption}
\usepackage{tabularx}
\usepackage{algorithm,algorithmic}
\usepackage{hhline}
\usepackage[normalem]{ulem}
\usepackage{amsthm}
\usepackage{bm}
\newtheorem{theorem}{Theorem}

\newtheorem{lemma}{Lemma}
\newtheorem{definition}{Definition}
\newtheorem{assumption}{Assumption}
\newtheorem*{definition*}{Definition}
\newtheorem{remark}{Remark}
\newcommand{\DrawSchemaFig}[4]{\subfigure[#2]{\label{#3}\includegraphics[bb=#4, keepaspectratio, height=\textheight, width=\textwidth]{#1}}}
\usepackage[toc,page]{appendix}
\usepackage{breakurl}
\date{}
\title{Few-shot Domain Adaptation by Causal Mechanism Transfer}
\begin{document}

\maketitle

\global\csname @topnum\endcsname 0

\begin{abstract}
We study \emph{few-shot supervised domain adaptation} (DA) for regression problems, where only a few labeled target domain data and many labeled source domain data are available.
Many of the current DA methods base their \emph{transfer assumptions} on either parametrized distribution shift or apparent distribution similarities, e.g., identical conditionals or small distributional discrepancies.
However, these assumptions may preclude the possibility of adaptation from intricately shifted and apparently very different distributions.
To overcome this problem, we propose \emph{mechanism transfer}, a meta-distributional scenario in which a data generating \emph{mechanism} is invariant across domains.
This transfer assumption can accommodate nonparametric shifts resulting in apparently different distributions while providing a solid statistical basis for DA.
We take the structural equations in causal modeling as an example and propose a novel DA method, which is shown to be useful both theoretically and experimentally.
Our method can be seen as the first attempt to fully leverage the structural causal models for DA.
\end{abstract}
\section{Introduction}
\label{sec:org77a861f}
Learning from a limited amount of data is a long-standing yet actively studied problem of machine learning.
Domain adaptation (DA) \cite{Ben-Davidtheory2010} tackles this problem by leveraging auxiliary data sampled from related but different domains.
In particular, we consider \emph{few-shot supervised} DA for regression problems, where only a few labeled target domain data and many labeled source domain data are available.

A key component of DA methods is the \emph{transfer assumption} (TA) to relate the source and the target distributions.
Many of the previously explored TAs have relied on certain direct distributional similarities, e.g., identical conditionals \cite{ShimodairaImproving2000} or small distributional discrepancies \cite{Ben-DavidAnalysis2007}.
However, these TAs may preclude the possibility of adaptation from apparently very different distributions.
Many others assume parametric forms of the distribution shift \cite{ZhangDomain2013} or the distribution family \cite{StorkeyMixture2007}
which can highly limit the considered set of distributions.
(we further review related work in Section~\ref{paper:sec:related-work-transfer-assumptions}).

To alleviate the intrinsic limitation of previous TAs due to relying on apparent distribution similarities or parametric assumptions,
we focus on a meta-distributional scenario where there exists a common generative \emph{mechanism} behind the data distributions (Figures~\ref{fig:schematic-illustration-1},\ref{fig:schematic-illustration-2}).
Such a common mechanism may be more conceivable in applications involving structured table data such as medical records \cite{YadavMining2018}.
For example, in medical record analysis for disease risk prediction, it can be reasonable to assume that there is a pathological mechanism that is common across regions or generations,
but the data distributions may vary due to the difference in cultures or lifestyles.
Such a hidden structure (pathological mechanism, in this case), once estimated, may provide portable knowledge to enable DA, allowing one to obtain accurate predictors for under-investigated regions or new generations.

Concretely, our assumption relies on the generative model of nonlinear independent component analysis (nonlinear ICA; Figure~\ref{fig:schematic-illustration-1}),
where the observed labeled data are generated by first sampling latent independent components (ICs) \(\S\) and later transforming them by a nonlinear invertible \emph{mixing function} denoted by \(\f\) \cite{HyvarinenNonlinear2019}.
\todothree{上と同様ですが、ICsが最初に nonparametric に生成されるという仮定は良くある話なんでしょうか}
Under this generative model, our TA is that \(\f\) representing the mechanism is identical across domains (Figure~\ref{fig:schematic-illustration-2}).
This TA allows us to formally relate the domain distributions and develop a novel DA method
without assuming their apparent similarities or making parametric assumptions.
\todothree{黒木さん：この段落に出てくるICAとfの関係がすぐに理解できない．．．}

\begin{figure}[!t]
\begin{minipage}[t]{1.0\linewidth}
\begin{minipage}[t]{0.3\linewidth}
\begin{figure}[H]
\begin{minipage}[t]{1.0\linewidth}\hspace*{\fill}
\begin{minipage}[c]{0.9\textwidth}
\includegraphics[keepaspectratio, bb=0 0 180 186, width=0.975\textwidth]{\figureRoot/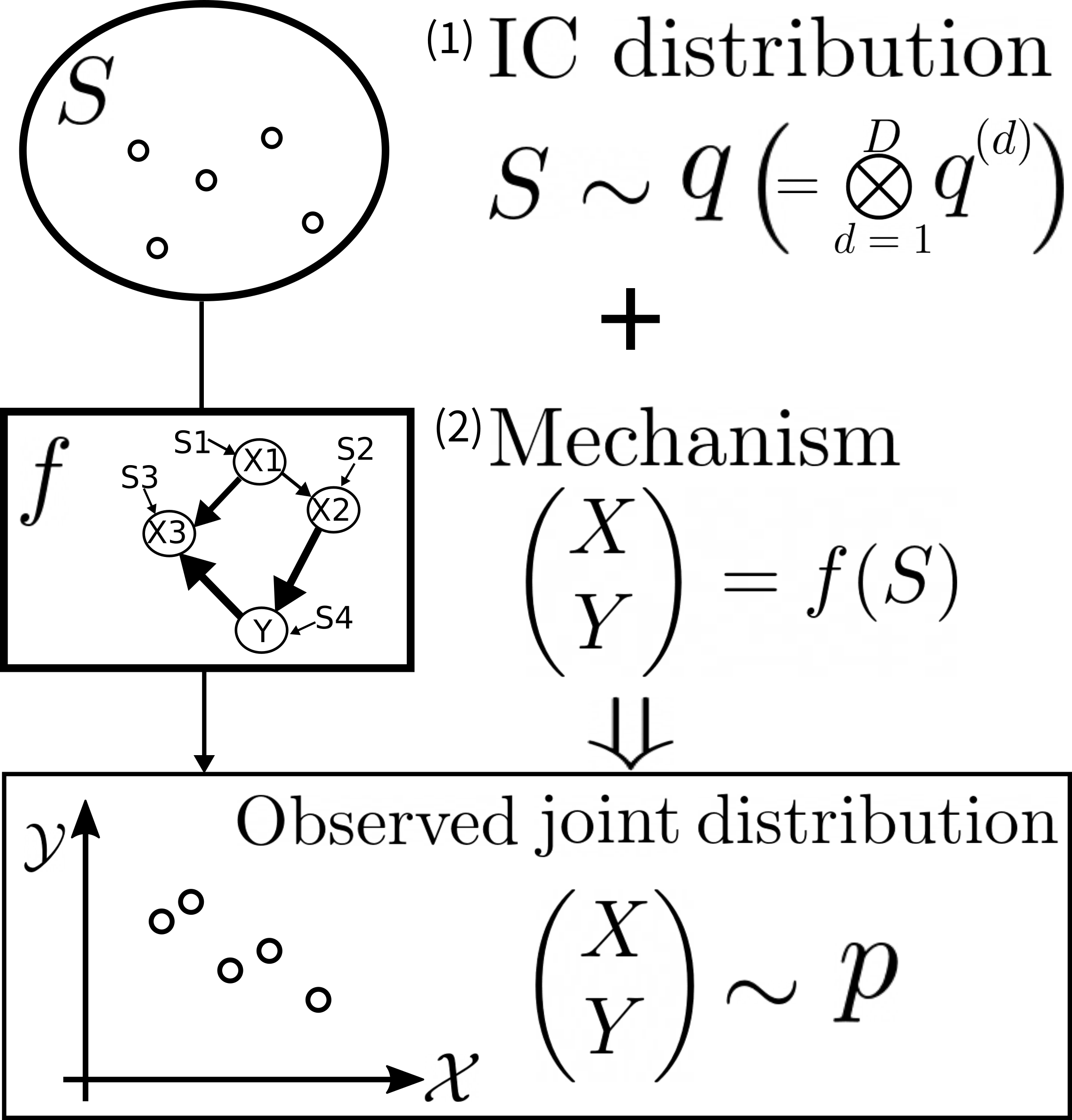}
\end{minipage}\hspace*{\fill}\vspace*{\fill}
\begin{minipage}[c]{1.0\linewidth}
\caption{
  Nonparametric generative model of nonlinear independent component analysis.
  Our meta-distributional transfer assumption is built on the model,
  where there exists an invertible function \(\f\) representing the mechanism
  to generate the labeled data \((X, Y)\) from the independent components (ICs), \(\S\), sampled from \todotwo{an independent distribution }\(\q\).
  As a result, each pair \((\f, \q)\) defines a joint distribution \(p\).
  \todothree{qは簡単に求められる分布という認識で正しいですか？}
  }
\label{fig:schematic-illustration-1}
\end{minipage}\vspace*{\fill}\hspace*{\fill}
\end{minipage}
\end{figure}
\end{minipage}\hspace*{\fill}
\begin{minipage}[t]{0.68\linewidth}
\begin{figure}[H]
\begin{minipage}[c]{1.0\linewidth}
\begin{minipage}[c]{1.0\linewidth}
\fbox{\includegraphics[keepaspectratio, bb=0 0 303 123, width=0.975\textwidth]{\figureRoot/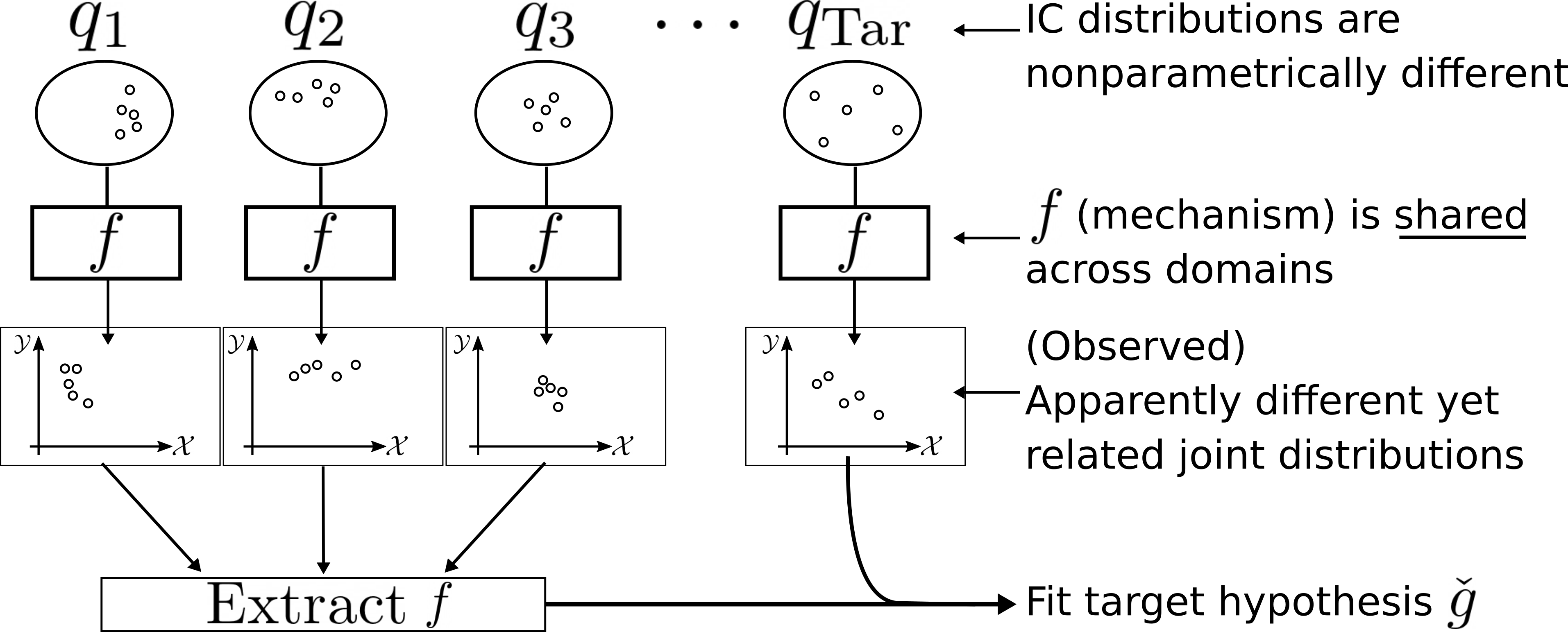}}
\vspace*{0.5\topsep}
\end{minipage}
\begin{minipage}[c]{1.0\linewidth}
\caption{
  Our assumption of common generative mechanism.
  By capturing the common data generation mechanism, we enable domain adaptation among seemingly very different distributions without relying on parametric assumptions.
  }
\label{fig:schematic-illustration-2}
\end{minipage}
\end{minipage}
\end{figure}
\end{minipage}
\end{minipage}
\end{figure}
\paragraph{Our contributions.}
\label{sec:orgbfe180d}
Our key contributions can be summarized in three points as follows.
\begin{enumerate}
\item We formulate the flexible yet intuitively accessible TA of shared generative mechanism
and develop a few-shot regression DA method (Section~\ref{paper:sec:proposed-method}).
The idea is as follows.
First, from the source domain data, we estimate the mixing function \(\f\) by nonlinear ICA \cite{HyvarinenNonlinear2019}
because \(\f\) is the only assumed relation of the domains.
Then, to transfer the knowledge, we perform data augmentation using the estimated \(\f\) on the target domain data using the independence of the IC distributions.
In the end, the augmented data is used to fit a target predictor (Figure~\ref{fig:schematic-illustration-algorithm}).
\item We theoretically justify the augmentation procedure by invoking the theory of generalized U-statistics \citep{LeeUStatistics1990}.
The theory shows that the proposed data augmentation procedure yields the uniformly minimum variance unbiased risk estimator in an ideal case.
We also provide an excess risk bound \cite{MohriFoundations2012} to cover a more realistic case (Section~\ref{paper:sec:theoretical-insights}).
\todotwo{, which involves slight extensions of the previously established learning theory for degree-\(2\) U-process minimization to higher degrees and to the von-Mises process minimization.}
\item We experimentally demonstrate the effectiveness of the proposed algorithm (Section~\ref{paper:sec:experiment}).
The real-world data we use is taken from the field of \emph{econometrics}, for which structural equation models have been applied in previous studies \cite{GreeneEconometric2012}.
\end{enumerate}

A salient example of the generative model we consider is the structural equations of causal modeling (Section~\ref{paper:sec:problem}).
In this context, our method can be seen as the first attempt to fully leverage the structural causal models for DA (Section~\ref{paper:sec:related-work-causality-and-transfer}).
\section{Problem Setup \label{paper:sec:problem}}
\label{sec:org90c380b}
In this section, we describe the problem setup and the notation.
To summarize, our problem setup is \emph{homogeneous}, \emph{multi-source}, and \emph{few-shot supervised} domain adapting regression.
That is, respectively, all data distributions are defined on the same data space, there are multiple source domains, and a limited number of labeled data is available from the target distribution
(and we do \emph{not} assume the availability of unlabeled data).
In this paper, we use the terms \emph{domain} and \emph{distribution} interchangeably.
\todothree{ここをどうにか読みやすくしたい}
\paragraph{Notation.}
\label{sec:org57b1cb1}
Let us denote the set of real (resp.\ natural) numbers by \(\mathbb{R}\) (resp.\ \(\mathbb{N}\)).
For \(N \in \mathbb{N}\), we define \([N] := \{1, 2, \ldots, N\}\).
\todo{To make it compact or not.}
Throughout the paper, we fix \(D (\in \mathbb{N}) > 1\) and suppose that the input space \(\Xsp\) is a subset of \(\mathbb{R}^{D-1}\) and the label space \(\Ysp\) is a subset of \(\mathbb{R}\).
As a result, the overall data space \(\Zsp := \Xsp \times \Ysp\) is a subset of \(\mathbb{R}^D\).
We generally denote a labeled data point by \(Z = (X, Y)\).
We denote by \(\Qsp\) the set of independent distributions on \(\mathbb{R}^D\) with absolutely continuous marginals.
For a distribution \(p\), we denote its induced expectation operator by \(\E_p\).
Table~\ref{tbl:notation} in Supplementary~Material provides a summary of notation.
\paragraph{Basic setup: Few-shot domain adapting regression.}
\label{sec:org61f2ef7}
Let \(\ptr\) be a distribution (the \emph{target distribution}) over \(\Zsp\), and let \(\G \subset \{\g: \ReDminus \to \Re\}\) be a hypothesis class.
Let \(\l : \G \times \ReD \to \LossValSp\) be a loss function where \(\LossUpperBound > 0\) is a constant.
Our goal is to find a predictor \(\g \in \G\) which performs well for \(\ptr\), i.e., the target risk \(\R(\g) := \Etr \l(\g, \Z)\) is small.
We denote \(\gstar \in \mathop{\rm arg~min}_{\g \in \G} \R(\g)\).
To this goal, we are given an independent and identically distributed (\IIDtext) sample \(\DataTar := \{\Z_i\}_{i=1}^{\ntr} \overset{\text{\IIDtext}}{\sim} \ptr\).
In a fully supervised setting where \(\ntr\) is large, a standard procedure is to select \(\g\) by empirical risk minimization (ERM), i.e., \(\hg \in \mathop{\rm arg~min}_{\g \in \G} \hR(\g)\), where \(\hR(\g) := \frac{1}{\ntr} \sum_{i=1}^{\ntr} \l(\g, \Z_i)\).
However, when \(\ntr\) is not sufficiently large, \(\hR(\g)\) may not accurately estimate \(\R(\g)\), resulting in a high generalization error of \(\hg\).
To compensate for the scarcity of data from the target distribution,
let us assume that we have data from \(K\) distinct \emph{source distributions} \(\{\psrck\}_{k=1}^K\) over \(\Zsp\), that is, we have independent \IID samples \(\Datak := \{\Zki\}_{i=1}^{\nk} \overset{\text{\IIDtext}}{\sim} \psrck (k \in [K], \nk \in \N)\) whose relations to \(\ptar\) are described shortly.
We assume \(\ntr, \nsrck \geq D\) for simplicity.
\todotwo{(We assume that there is enough number of distinct accessible domains. This is because our proposed method relies on the identifiability results of nonlinear ICA based on contrastive learning. Depending on the identification condition, our proposed approach can be extended to other situations, e.g., single-source domain adaptation is possible if the problem is reduced to linear ICA.)}
\paragraph{Key assumption.   \label{paper:sec:example-sem}}
\label{sec:org1307d56}
\todothree{ここでFigure ２を参照してもqの実態を理解できなかった}
In this work, the key transfer assumption is that all domains follow nonlinear ICA models with identical mixing functions (Figure~\ref{fig:schematic-illustration-2}).
To be precise, we assume that there exists a set of IC distributions \(\qtar, \qsrck \in \Qsp (k \in [K])\) , and a smooth invertible function \(\f: \ReD \to \ReD\) (the \emph{transformation} or \emph{mixing})
such that \(\Zki \sim \psrck\) is generated by first sampling \(\Ski \sim \qsrck\) and later transforming it by
\todo{伝わるか？曖昧性は無いか？ICAの論文を見る}
\begin{equation}\label{paper:eq:nonlinear-mixing}
\Zki = \f(\Ski),
\end{equation}
and similarly \(\Zi = \f(\Si), \Si \sim \qtar\) for \(\ptar\).
The above assumption allows us to formally relate \(\psrck\) and \(\ptar\).
It also allows us to estimate \(\f\) when sufficient identification conditions required by the theory of nonlinear ICA are met\todothree{conditionは一般的なのか結構強いのか，くらいはきになるなあという気持ちに}.
Due to space limitation, we provide a brief review of the nonlinear ICA method used in this paper and the known theoretical conditions in \Supplementary~\ref{paper:sec:appendix:nonlinear-ica-gcl}.
Having multiple source domains is assumed here for the identifiability of \(\f\): it comes from the currently known identification condition of nonlinear ICA \cite{HyvarinenNonlinear2019}.
Note that complex changes in \(\q\) are allowed, hence the assumption of invariant \(\f\) can accommodate intricate shifts in the apparent distribution \(\p\).
We discuss this further in Section~\ref{paper:sec:discussion:flexible-dist-shift} by taking a simple example.
\paragraph{Example: Structural equation models}
\label{sec:org3620df9}
A salient example of generative models expressed as Eq.~\eqref{paper:eq:nonlinear-mixing} is \emph{structural equation models} (SEMs; \citealp{PearlCausality2009,PetersElements2017}), which are used to describe the data generating mechanism involving the causality of random variables \citep{PearlCausality2009}.
More precisely, the generative model of Eq.\eqref{paper:eq:nonlinear-mixing} corresponds to the \emph{reduced form} \cite{ReissStructural2007} of a \emph{Markovian} SEM \cite{PearlCausality2009}, i.e., a form where the structural equations to determine \(Z\) from \((Z, S)\) are solved so that \(Z\) is expressed as a function of \(S\).
Such a conversion is always possible because a Markovian SEM induces an \emph{acyclic} causal graph \cite{PearlCausality2009}, hence the structural equations can be solved by elimination of variables.
This interpretation of reduced-form SEMs as Eq.\eqref{paper:eq:nonlinear-mixing} has been exploited in methods of \emph{causal discovery}, e.g., in the linear non-Gaussian additive-noise models and their successors \citep{KanoCausal2003,Shimizulinear2006,MontiCausal2019b}.
In the case of SEMs, the key assumption of this paper translates into the invariance of the structural equations across domains, which enables an intuitive assessment of the assumption based on prior knowledge. For instance, if all domains have the same causal mechanism and are in the same intervention state (including an intervention-free case), the modeling choice is deemed plausible.
Note that we do not estimate the original structural equations in the proposed method (Section~\ref{paper:sec:proposed-method}) but we only require estimating the reduced form, which is an easier problem compared to causal discovery.
\section{Proposed Method: Mechanism Transfer \label{paper:sec:proposed-method}}
\label{sec:orgf02327e}
In this section, we detail the proposed method, mechanism transfer (Algorithm~\ref{paper:alg:proposed-method}).
The method first estimates the common generative mechanism \(\f\) from the source domain data and then uses it\todotwo{the estimated mechanism} to perform data augmentation of the target domain data to transfer the knowledge (Figure~\ref{fig:schematic-illustration-algorithm}).

\todotwo{Let the readers naturally understand why the procedure is derived.}

\begin{figure}[!t]
\begin{minipage}[t]{1.0\linewidth}
\begin{minipage}[t]{0.48\linewidth}
\begin{algorithm}[H]
\caption{Proposed method: mechanism transfer}\label{paper:alg:proposed-method}
\begin{algorithmic}
\renewcommand{\algorithmicrequire}{\textbf{Input:}}
\renewcommand{\algorithmicensure}{\textbf{Output:}}
\REQUIRE Source domain data sets \(\{\Datak\}_{k \in [K]}\), target domain data set \(\DataTar\),
nonlinear ICA algorithm \(\NLICA\), and a learning algorithm \(\mathcal{A}_{\G}\) to fit the hypothesis class \(\G\) of predictors.
\STATE // Step 1. Estimate the shared transformation.
\STATE \quad \(\hf \gets \NLICA(\Datakk{1}, \ldots, \Datakk{K})\)
\STATE // Step 2. Extract and shuffle target independent components
\STATE \quad \(\hat \e_i \gets \hfinv(\Z_i), \quad (i = 1, \ldots, \ntr)\)
\STATE \quad \(\{\augSi\}_{\augi\in[\ntr]^D} \gets \AllCombinations(\{\hat \e_i\}_{i=1}^{\ntr})\)\todo{コンビネーション全部取るのか問題}
\STATE // Step 3. Synthesize target data and fit the predictor.
\STATE \quad \(\augZi \gets \hf(\augSi)\)
\STATE \quad \(\cbarhg \gets \mathcal{A}_{\G}(\{\augZi\}_{\augi})\)
\ENSURE \(\cbarhg\): the predictor in the target domain.
\end{algorithmic}\end{algorithm}
\end{minipage}\hfill
\begin{minipage}[t]{0.48\linewidth}
\begin{figure}[H]
\begin{minipage}[c]{1.0\linewidth}
\begin{minipage}[c]{1.0\linewidth}
\begin{minipage}[c]{0.2\linewidth}
\DrawSchemaFig{\figureRoot/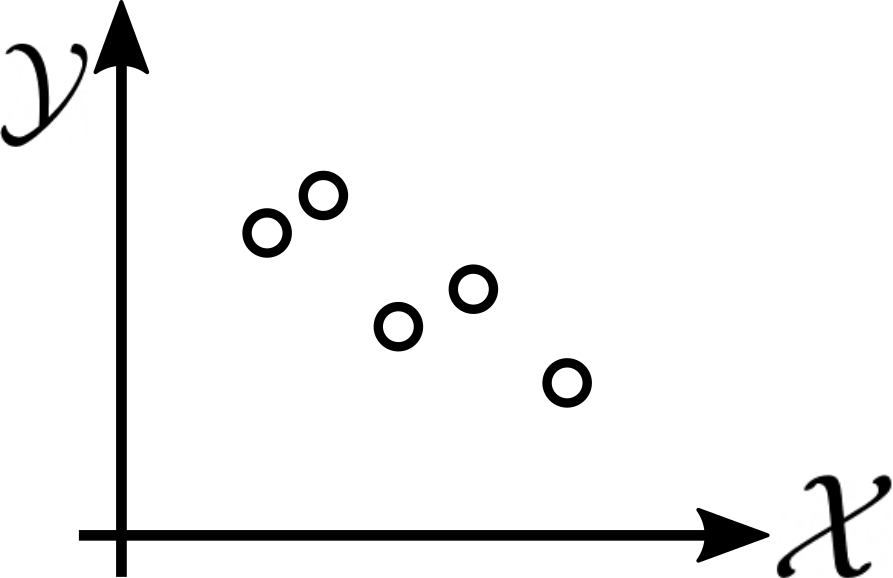}{Labeled target data}{fig:schema-1}{0 0 71 46}
\end{minipage}\hfill\(\overset{\hfinv}{\to}\)\hfill
\begin{minipage}[c]{0.2\linewidth}
\DrawSchemaFig{\figureRoot/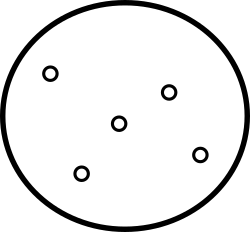}{Find IC}{fig:schema-2}{0 0 71 46}
\end{minipage}\hfill\(\to\)\hfill
\begin{minipage}[c]{0.2\linewidth}
\DrawSchemaFig{\figureRoot/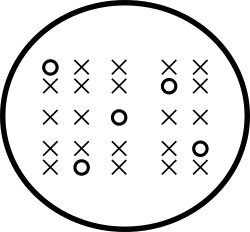}{Shuffle}{fig:schema-3}{0 0 71 46}
\end{minipage}\hfill\(\overset{\hf}{\to}\)\hfill
\begin{minipage}[c]{0.2\linewidth}
\DrawSchemaFig{\figureRoot/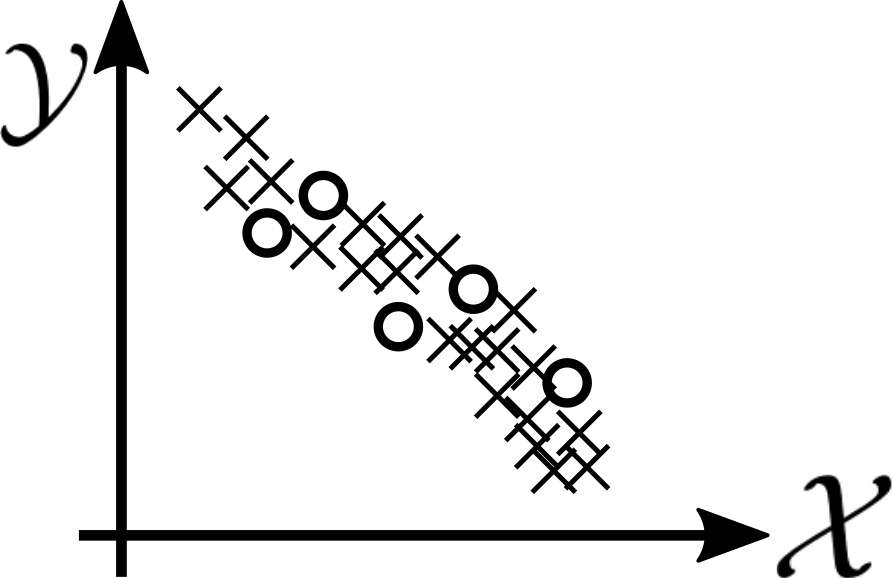}{Pseudo target data}{fig:schema-4}{0 0 71 46}
\end{minipage}\hfill
\end{minipage}
\begin{minipage}[c]{1.0\linewidth}
\caption{
  Schematic illustration of proposed few-shot domain adaptation method after estimating the common mechanism \(\f\)\todotwo{Shuffle -> Exchange}.
  With the estimated \(\hf\), the method augments the small target domain sample in a few steps to enhance statistical efficiency:
  \subref{fig:schema-1} The algorithm is given labeled target domain data.
  \subref{fig:schema-2} From labeled target domain data, extract the ICs.
  \subref{fig:schema-3} By shuffling the values, synthesize likely values of IC.
  \subref{fig:schema-4} From the synthesized IC, generate pseudo target data.
  The generated data is used to fit a predictor for the target domain.
  \todotwo{ここで図のキャプションの付け方を変えて，(a)を消し，b〜dは矢印に合わせる．するとキャプションの説明でも(a)を省けてスペースが作れる．}
  }
\label{fig:schematic-illustration-algorithm}
\end{minipage}
\end{minipage}
\end{figure}
\end{minipage}
\end{minipage}
\end{figure}
\subsection{Step 1: Estimate \(\f\) using the source domain data}
\label{sec:org01994ff}
The first step estimates the common transformation \(\f\) by nonlinear ICA, namely via \emph{generalized contrastive learning} (GCL; \citealp{HyvarinenNonlinear2019}).
GCL uses auxiliary information for training a certain binary classification function, \(\rhf\), equipped with a parametrized feature extractor \(\hf: \ReD \to \ReD\).
The trained feature extractor \(\hf\) is used as an estimator of \(\f\).
The auxiliary information we use in our problem setup is the domain indices \([K]\).
The classification function to be trained in GCL is \(\rhf(\z, \u) := \sum_{d=1}^D \psi_d(\hfinv(\z)_d, \u)\)
consisting of \((\hf, \{\psi_d\}_{d=1}^D)\), and the classification task of GCL is logistic regression to classify \((\Zk, k)\) as positive and \((\Zk, k') (k' \neq k)\) as negative.
This yields the following domain-contrastive learning criterion to estimate \(\f\):
\begin{equation*}\begin{split}
\argmin_{\substack{\hf \in \F, \\ \{\psi_d\}_{d=1}^D \subset \NNPsid}} \sum_{k=1}^K \frac{1}{\nsrck} \sum_{i=1}^{\nsrck} \biggl(&\phi\left(\rhf(\Zki, k)\right) + \mathbb{E}_{k' \neq k}\phi\left(-\rhf(\Zki, k')\right) \biggr),
\end{split}\end{equation*}
where \(\F\) and \(\NNPsid\) are sets of parametrized functions,
\(\mathbb{E}_{k'\neq k}\) denotes the expectation with respect to \(k' \sim \Unif([K] \setminus \{k\})\) (\(\Unif\) denotes the uniform distribution),
and \(\phi\) is the logistic loss \(\phi(m) := \log(1 + \exp(- m))\).
We use the solution \(\hf\) as an estimator of \(\f\).
In experiments, \(\F\) is implemented by invertible neural networks \cite{KingmaGlow2018}, \(\NNPsid\) by multi-layer perceptron, and \(\mathbb{E}_{k'\neq k}\) is replaced by a random sampling renewed for every mini-batch.
\subsection{Step 2: Extract and inflate the target ICs using \(\hf\)}
\label{sec:orgeab105d}
The second step extracts and inflates the target domain ICs using the estimated \(\hf\).
We first extract the ICs of the target domain data by applying the inverse of \(\hf\) as
\begin{equation*}\begin{split}
\hat \s_i = \hfinv(\Z_i).
\end{split}\end{equation*}
After the extraction, we inflate the set of IC values by taking all dimension-wise combinations of the estimated IC:
\begin{equation*}\begin{split}
\augSi = (\hat \s_{i_1}^{(1)}, \ldots, \hat \s_{i_D}^{(D)}), \quad {\bm{i}} = (i_1, \ldots, i_D) \in [\ntr]^D,
\end{split}\end{equation*}
to obtain new plausible IC values \(\augSi\).
The intuitive motivation of this procedure stems from the independence of the IC distributions.
Theoretical justifications are provided in Section~\ref{paper:sec:theoretical-insights}.
In our implementation, we use invertible neural networks \cite{KingmaGlow2018} to model the function \(\hf\) to enable the computation of the inverse \(\hfinv\).
\subsection{Step 3: Synthesize target data from the inflated ICs}
\label{sec:orgd4645d1}
The third step estimates the target risk \(\R\) by the empirical distribution of the augmented data:
\begin{equation}\label{paper:eq:augmented-erm}\begin{split}
\cbarhR(\g) := \frac{1}{\ntr^D}\sum_{\bm{i} \in [\ntr]^D} \l(\g, \hf(\augSi)),
\end{split}\end{equation}
and performs empirical risk minimization.
In experiments, we use a regularization term \(\Omega(\cdot)\) to control the complexity of \(\G\) and select
\begin{equation*}\begin{split}
\cbarhg \in \argming \left\{\cbarhR(\g) + \Omega(\g)\right\}.
\end{split}\end{equation*}
The generated hypothesis \(\cbarhg\) is then used to make predictions in the target domain.
In our experiments, we use \(\Omega(\g) = \lambda \|\g\|^2\), where \(\lambda > 0\) and the norm is that of the reproducing kernel Hilbert space (RKHS) which we take the subset \(\G\) from.
Note that we may well subsample only a subset of combinations in Eq.~\eqref{paper:eq:augmented-erm} to mitigate the computational cost similarly to \citet{ClemenconScalingup2016} and \citet{PapaSGD2015}.
\section{Theoretical Insights \label{paper:sec:theoretical-insights}}
\label{sec:org7d16cfb}
In this section, we state two theorems to investigate the statistical properties of the method proposed in Section~\ref{paper:sec:proposed-method} and provide plausibility beyond the intuition that we take advantage of the independence of the IC distributions.
\subsection{Minimum variance property: Idealized case}
\label{sec:org84ac3f0}
The first theorem provides an insight into the statistical advantage of the proposed method: in the ideal case, the method attains the minimum variance among all possible unbiased risk estimators.

\begin{theorem}[Minimum variance property of \(\cbarhR\)]
Assume that \(\hf = \f\).
Then, for each \(\g \in \G\), the proposed risk estimator \(\cbarhR(\g)\) is the uniformly minimum variance unbiased estimator of \(\R(\g)\), i.e.,
for any unbiased estimator \(\tilde \R(\g)\) of \(\R(\g)\),
\begin{equation*}\begin{split}
\forall \q \in \Qsp, \quad \Var(\cbarhR(\g)) \leq \Var(\tilde \R(\g))
\end{split}\end{equation*}
as well as \(\Etr \cbarhR(\g) = \R(\g)\) holds.
\label{paper:thm:1}
\end{theorem}
The proof of Theorem~\ref{paper:thm:1} is immediate once we rewrite \(\R(\g)\) as a \(D\)-variate regular statistical functional and
\(\cbarhR(\g)\) as its corresponding generalized U-statistic \citep{LeeUStatistics1990}. Details can be found in Supplementary~Material~\ref{paper:sec:appendix:theory-1}.
Theorem~\ref{paper:thm:1} implies that the proposed risk estimator can have superior statistical efficiency in terms of the variance over the ordinary empirical risk.
\subsection{Excess risk bound: More realistic case}
\label{sec:org5dd201d}
In real situations, one has to estimate \(\f\).
The following theorem characterizes the statistical gain\todo{これは？} and loss arising from the estimation error \(\f - \hf\).
The intuition is that the increased number of points suppresses the possibility of overfitting because the hypothesis has to fit the majority of the inflated data,
but the estimator \(\hf\) has to be accurate so that fitting to the inflated data is meaningful.
Note that the theorem is agnostic to how \(\hf\) is obtained, hence it applies to more general problem setup as long as \(\f\) can be estimated.

\begin{theorem}[Excess risk bound]
Let \(\cbarhg\) be a minimizer of Eq.~\eqref{paper:eq:augmented-erm}.
Under appropriate assumptions (see Theorem~\ref{thm:generalization-error} in Supplementary~Material),
for arbitrary \(\delta, \delta' \in (0, 1)\), we have with probability at least \(1 - (\delta + \delta')\),
\begin{equation*}\begin{split}
\R(\cbarhg) - \R(\gstar) \leq \annot{\ApproxErrorUpperBoundLeading}{Approximation error} + \annot{\PseudoGeneralizationErrorBoundForFinalErrorBoundOne}{Estimation error} + \annot{\finalErrorBoundSecondHigherOrder + \ApproxErrorUpperBoundResidual}{Higher order terms}.
\end{split}\end{equation*}
Here, \(\SobolevOne{\cdot}\) is the \((1, 1)\)-Sobolev norm, and we define the effective Rademacher complexity \(\Rad(\G)\) by
\begin{equation}\label{paper:eq:rademacher-def}\begin{split}
\Rad(\G) := \frac{1}{n} \mathbb{E}_{\hat \S}\Erad\left[\supg \left|\sum_{i=1}^{n} \rad_i \ETwoToD{\S'}[\ke(\hat \s_i, \S_2', \ldots, \S_D')]\right|\right],
\end{split}\end{equation}
where \(\{\rad_i\}_{i=1}^n\) are independent sign variables,
\(\mathbb{E}_{\hat \S}\) is the expectation with respect to \(\{\hat \s_i\}_{i=1}^{\ntr}\),
the dummy variables \(S_2', \ldots, S_D'\) are \IID copies of \(\hat\s_1\),
and \(\ke\) is defined by using the degree-\(D\) symmetric group \(\DSymmetric\) as
\begin{equation*}\begin{split}
\ke(\e_1, \ldots, \e_D) := \frac{1}{D!} \sum_{\pi \in \DSymmetric} \l(\g, \hf(\e_{\pi(1)}^{(1)}, \ldots, \e_{\pi(D)}^{(D)})),
\end{split}\end{equation*}
and \(\finalErrorBoundSecondHigherOrder\) and \(\ApproxDensityDifferenceBoundRemainder\) are higher order terms.
The constants \(B_q\) and \(B_\l\) depend only on \(\q\) and \(\l\), respectively, while \(C'\) depends only on \(\f, \q, \l\), and \(D\).
\label{paper:thm:2}
\end{theorem}
Details of the statement and the proof can be found in Supplementary~Material~\ref{paper:sec:appendix:theory-2}.
The Sobolev norm \cite{AdamsSobolev2003} emerges from the evaluation of the difference between the estimated IC distribution and the ground-truth IC distribution.
In Theorem~\ref{paper:thm:2}, the utility of the proposed method appears in the effective complexity measure.
The complexity is defined by a set of functions which are marginalized over all but one argument, resulting in mitigated dependence on the input dimensionality from exponential to linear
(\Supplementary~\ref{paper:sec:appendix:theory-2}, Remark~\ref{theory:remark:comparison-of-rademacher}).
\section{Related Work and Discussion \label{paper:sec:related-work}}
\label{sec:org93f0298}
In this section, we review some existing TAs for DA to clarify the relative position of the paper.
We also clarify the relation to the literature of causality-related transfer learning.
\subsection{Existing transfer assumptions \label{paper:sec:related-work-transfer-assumptions}}
\label{sec:org903e0df}
Here, we review some of the existing work and TAs.
See Table~\ref{tbl:relative-position} for a summary.
\paragraph{(1) Parametric assumptions.}
\label{sec:org6a00533}
Some TAs assume parametric distribution families, e.g., Gaussian mixture model in covariate shift \cite{StorkeyMixture2007}.
Some others assume parametric distribution shift, i.e., parametric representations of the target distribution given the source distributions.
Examples include location-scale transform of class conditionals \cite{ZhangDomain2013,GongDomain2016b},
linearly dependent class conditionals \cite{ZhangMultisource2015},
and low-dimensional representation of the class conditionals after kernel embedding \cite{StojanovDatadriven2019}.
In some applications, e.g., remote sensing, some parametric assumptions have proven useful \cite{ZhangDomain2013}.
\paragraph{(2) Invariant conditionals and marginals.}
\label{sec:org38e9e5a}
Some methods assume invariance of certain conditionals or marginals \cite{Quinonero-CandelaDataset2009}, e.g., \(p(Y|X)\) in the covariate shift scenario \cite{ShimodairaImproving2000},
\(p(Y|\mathcal{T}(X))\) for an appropriate feature transformation \(\mathcal{T}\) in transfer component analysis \cite{PanDomain2011},
\(p(Y|\mathcal{T}(X))\) for a feature selector \(\mathcal{T}\) \cite{Rojas-CarullaInvariant2018,MagliacaneDomain2018},
\(p(X|Y)\) in the target shift (TarS) scenario \cite{ZhangDomain2013,NguyenContinuous2016},
and few components of regular-vine copulas and marginals in \citet{Lopez-pazSemisupervised2012}.
For example, the covariate shift scenario has been shown to fit well to brain computer interface data \cite{SugiyamaCovariate2007}.
\paragraph{(3) Small discrepancy or integral probability metric.}
\label{sec:org2ffacc5}
Another line of work relies on certain distributional similarities,
e.g., integral probability metric \cite{CourtyJoint2017a}
or hypothesis-class dependent discrepancies \cite{Ben-DavidAnalysis2007,BlitzerLearning2008,Ben-Davidtheory2010,KurokiUnsupervised2019,ZhangBridging2019,CortesAdaptation2019}.
These methods assume the existence of the \emph{ideal joint hypothesis} \cite{Ben-Davidtheory2010}, corresponding to a relaxation of the covariate shift assumption.
These TA are suited for unsupervised or semi-supervised DA in computer vision applications \cite{CourtyJoint2017a}.
\todotwo{ここにBen-Davidの不可能性を引用する?}
\paragraph{(4) Transferable parameter.}
\label{sec:orgafac0a9}
Some others consider parameter transfer \cite{KumagaiLearning2016a},
where the TA is the existence of a parameterized feature extractor that performs well in the target domain for linear-in-parameter hypotheses and its learnability from the source domain data.
For example, such a TA has been known to be useful in natural language processing or image recognition \cite{LeeExponential2009,KumagaiLearning2016a}.
\begin{table}
\centering
\caption{
Comparison of TAs for DA
(\emph{Parametric}: parametric distribution family or distribution shift,
\emph{Invariant dist.}: invariant distribution components such as conditionals, marginals, or copulas.
\emph{Disc.} / \emph{IPM}: small discrepancy or integral probability metric,
\emph{Param-transfer}: existence of transferable parameter,
\emph{Mechanism}: invariant mechanism).
AD: adaptation among Apparently Different distributions is accommodated.
NP: Non-Parametrically flexible.
BCI: Brain computer interface.
The numbers indicate the paragraphs of Section~\ref{paper:sec:related-work-transfer-assumptions}.
}
\label{tbl:relative-position}

\begin{center}
\begin{tabular}{lccl}
\hline
TA & AD & NP & Suited app. example\\
\hline
(1) Parametric & \(\checkmark\) & - & Remote sensing\\
(2) Invariant dist. & - & \(\checkmark\) & BCI\\
(3) Disc. / IPM & - & \(\checkmark\) & Computer vision\\
(4) Param-transfer & \(\checkmark\) & \(\checkmark\) & Computer vision\\
(Ours) Mechanism & \(\checkmark\) & \(\checkmark\) & Medical records\\
\hline
\end{tabular}
\end{center}
\end{table}
\subsection{Causality for transfer learning \label{paper:sec:related-work-causality-and-transfer}}
\label{sec:org3570321}
Our method can be seen as the first attempt to fully leverage structural causal models for DA.
Most of the causality-inspired DA methods express their assumptions in the level of \emph{graphical causal models} (GCMs), which only has much coarser information than \emph{structural causal models} (SCMs) \citep[Table~1.1]{PetersElements2017} exploited in this paper. Compared to previous work, our method takes one step further to assume and exploit the invariance of SCMs.
Specifically, many studies assume the GCM \(X \gets Y\) (the \emph{anticausal} scenario) following the seminal meta-analysis of \citet{Scholkopfcausal2012a} and use it to motivate
their parametric distribution shift assumptions or the parameter estimation procedure \cite{ZhangDomain2013,ZhangMultisource2015,GongDomain2016b,GongCausal2018}.
Although such assumptions on the GCM have the virtue of being more robust to misspecification, they tend to require parametric assumptions to obtain theoretical justifications.
On the other hand, our assumption enjoys a theoretical guarantee without relying on parametric assumptions.

One notable work in the existing literature is \citet{MagliacaneDomain2018} that considered the domain adaptation among \emph{different intervention states}, a problem setup that complements ours that considers an intervention-free (or identical intervention across domains) case.
To model intervention states, \citet{MagliacaneDomain2018} also formulated the problem setup using SCMs, similarly to the present paper.
Therefore, we clarify a few key differences between \citet{MagliacaneDomain2018} and our work here.
In terms of the methodology, \citet{MagliacaneDomain2018} takes a variable selection approach to select a set of predictor variables with an invariant conditional distribution across different intervention states.
On the other hand, our method estimates the SEMs (in the reduced form) and applies a data augmentation procedure to transfer the knowledge.
To the best of our knowledge, the present paper is the first to propose a way to directly use the estimated SEMs for domain adaptation, and the fine-grained use of the estimated SEMs enables us to derive an excess risk bound.
In terms of the plausible applications, their problem setup may be more suitable for application fields with interventional experiments such as genomics, whereas ours may be more suited for fields where observational studies are more common such as health record analysis or economics.
In Appendix~\ref{paper:sec:appendix:further-related-work}, we provide a more detailed comparison.

\subsection{Plausibility of the assumptions \label{paper:sec:discussion:flexible-dist-shift}}
\label{sec:org32e2c9a}
\paragraph{Checking the validity of the assumption.}
\label{sec:orgdea9c66}
As is often the case in DA, the scarcity of data disables data-driven testing of the TAs, and we need domain knowledge to judge the validity.
For our TA, the intuitive interpretation as invariance of causal models (Section~\ref{paper:sec:problem}) can be used.
\paragraph{Invariant causal mechanisms.}
\label{sec:org2fa3289}
The invariance of causal mechanisms has been exploited in recent work of causal discovery such as \citet{XupoolingLiNGAM2014} and \citet{MontiCausal2019b}, or under the name of the \emph{multi-environment setting} in \citet{GhassamiLearning2017a}. Moreover, the SEMs are normally assumed to remain invariant unless explicitly intervened in \cite{HunermundCausal2019}.
However, the invariance assumption presumes that the intervention states do not vary across domains (allowing for the intervention-free case),
which can be limiting for some applications where different interventions are likely to be present, e.g., different treatment policies being put in place in different hospitals.
Nevertheless, the present work can already be of practical interest if it is combined with the effort to find suitable data or situations.
For instance, one may find medical records in group hospitals where the same treatment criteria is put in place or local surveys in the same district enforcing identical regulations.
In future work, relaxing the requirement to facilitate the data-gathering process is an important area.
For such future extensions, the present theoretical analyses can also serve as a landmark to establish what can be guaranteed in the basic case without mechanism alterations.
\paragraph{Fully observed variables.}
\label{sec:org5b614ea}
As the first algorithm in the approach to fully exploit SCMs for DA, we also consider the case where all variables are observable.
Although it is often assumed in a causal inference problem that there are some unobserved confounding variables, we leave further extension to such a case for future work.
\paragraph{Required number of source domains.}
\label{sec:org1fcac64}
A potential drawback of the proposed method is that it requires a number of source domains in order to satisfy the identification condition of the nonlinear ICA, namely GCL in this paper (\Supplementary~\ref{paper:sec:appendix:nonlinear-ica-gcl}).
The requirement solely arises from the identification condition of the ICA method and therefore has the possibility to be made less stringent by the future development of nonlinear ICA methods.
Moreover, if one can accept other identification conditions, one-sample ICA methods (e.g., linear ICA) can be also used in the proposed approach in a straightforward manner, and our theoretical analyses still hold regardless of the method chosen.
\paragraph{Flexibility of the model.}
\label{sec:org7460205}
The relation between \(X\) and \(Y\) can drastically change while \(\f\) is invariant.
For example, even in a simple additive noise model \((X, Y) = f(\S_1, \S_2) = (\S_1, \S_1 + \S_2)\),
the conditional \(p(Y|X)\) can shift drastically if the distribution of the independent noise \(\S_2\) changes in a complex manner, e.g., becoming multimodal from unimodal.
\section{Experiment \label{paper:sec:experiment}}
\label{sec:orgf51426f}
In this section, we provide proof-of-concept experiments to demonstrate the effectiveness of the proposed approach.
Note that the primary purpose of the experiments is to confirm whether the proposed method can properly perform DA in real-world data,
and it is not to determine which DA method and TA are the most suited for the specific dataset.
\subsection{Implementation details of the proposed method \label{paper:sec:implementation-proposed}}
\label{sec:org8eeb9eb}
\paragraph{Estimation of \(\f\) (Step 1).}
\label{sec:orgf988091}
We model \(\hf\) by an \(8\)-layer Glow neural network (\Supplementary~\ref{paper:sec:appendix:glow}).
We model \(\psi_d\) by a \(1\)-hidden-layer neural network with a varied number of hidden units, \(K\) output units, and the rectified linear unit activation \cite{LeCunDeep2015}.
We use its \(k\)-th output (\(k \in [K]\)) as the value for \(\psi_d(\cdot, k)\).
For training, we use the Adam optimizer \cite{KingmaAdam2017} with fixed parameters \((\beta_1, \beta_2, \epsilon) = (0.9, 0.999, 10^{-8})\), fixed initial learning rate \(10^{-3}\), and the maximum number of epochs \(300\).
The other fixed hyperparameters of \(\hf\) and its training process are described in \Supplementary~\ref{paper:sec:appendix:experiment-detail}.
\paragraph{Augmentation of target data (Step 3).}
\label{sec:orgc198548}
For each evaluation step, we take all combinations (with replacement) of the estimated ICs to synthesize target domain data.
After we synthesize the data, we filter them by applying a novelty detection technique with respect to the union of source domain data.
Namely, we use one-class support vector machine \cite{ScholkopfSupport2000} with the fixed parameter \(\nu=0.1\) and radial basis function (RBF) kernel \(k(x, y) = \exp(- \|x - y\|^2/\gamma)\) with \(\gamma = D\).
This is because the estimated transform \(\hf\) is not expected to be trained well outside the union of the supports of the source distributions.
After performing the filtration, we combined the original target training data with the augmented data to ensure the original data points to be always included.
\paragraph{Predictor hypothesis class \(\G\).}
\label{sec:orgf4c389d}
As the predictor model, we use the kernel ridge regression (KRR) with RBF kernel.
The bandwidth \(\gamma\) is chosen by the median heuristic similarly to \citet{YamadaRelative2011} for simplicity\todotwo{Details of the implementation: which median? Source only? Use Target data? Augmented target, too?}.
Note that the choice of the predictor model is for the sake of comparison with the other methods tailored for KRR \cite{CortesAdaptation2019},
and that an arbitrary predictor hypothesis class and learning algorithm can be easily combined with the proposed approach.
\paragraph{Hyperparameter selection.}
\label{sec:orge658a4b}
We perform grid-search for hyperparameter selection.
The number of hidden units for \(\psi_d\) is chosen from \(\{10, 20\}\) and the coefficient of weight-decay from \(10^{\{-2, -1\}}\).
The \(\ltwoSp\) regularization coefficient \(\lambda\) of KRR is chosen from \(\lambda \in 2^{\{-10, \ldots, 10\}}\) following \citet{CortesAdaptation2019}.
To perform hyperparameter selection as well as early-stopping, we record the leave-one-out cross-validation (LOOCV) mean-squared error on the target training data every \(20\) epochs and select its minimizer.
The leave-one-out score is computed using the well-known\todotwo{(Cite.)} analytic formula instead of training the predictor for each split.
Note that we only use the original target domain data as the held-out set and not the synthesized data.
In practice, if the target domain data is extremely few, one may well use \emph{percentile-cv} \cite{NgPreventing1997} to mitigate overfitting of hyperparameter selection.
\paragraph{Computation environment}
\label{sec:org006d701}
All experiments were conducted on an \todotwo{single processor of }Intel Xeon(R) \unitnum{2.60}{GHz} CPU with \unitnum{132}{GB} memory. They were implemented in Python using the PyTorch library \cite{PaszkePyTorch2019} or the R language \cite{RCoreTeamlanguage2018}.
\todo{and R for copula.}
\todotwo{For each hyperparameter configuration, the proposed method took at most \todo{(8 minutes)} for training and evaluation.}
\subsection{Experiment using real-world data}
\label{sec:org0b7f394}
\paragraph{Dataset.}
\label{sec:org4939653}
We use the gasoline consumption data \citep[p.284, Example~9.5]{GreeneEconometric2012},
which is a panel data of gasoline usage in 18 of the OECD countries over 19 years\todotwo{ (the total number of observations is 342)}.
We consider each country as a domain, and we disregard the time-series structure and consider the data as \IID samples for each country in this proof-of-concept experiment.
The dataset contains four variables, all of which are log-transformed: motor gasoline consumption per car (the predicted variable), per-capita income, motor gasoline price, and the stock of cars per capita (the predictor variables) \cite{BaltagiGasoline1983}.
For further details of the data, see Supplementary~Material~\ref{paper:sec:appendix:experiment-detail}.
We used the dataset because there are very few public datasets for domain adapting regression tasks \cite{CortesDomain2014} especially for multi-source DA,
and also because the dataset has been used in econometric analyses involving SEMs \cite{BaltagiEconometric2005}, conforming to our approach\todotwo{(Except for the invertibility assumption.)}.
\paragraph{Compared methods.}
\label{sec:orge0e1364}
We compare the following DA methods, all of which apply to regression problems.
Unless explicitly specified, the predictor class \(\G\) is chosen to be KRR with the same hyperparameter candidates as the proposed method (Section~\ref{paper:sec:implementation-proposed}).
Further details are described in \Supplementary~\ref{paper:sec:appendix:compared-methods-details}.
\todotwo{(Explain why these methods were chosen.)}
\begin{itemize}
\item Naive baselines (\emph{SrcOnly}, \emph{TarOnly}, and \emph{S\&TV}): \emph{SrcOnly} (resp. \emph{TarOnly}) trains a predictor on the source domain data (resp. target training data) without any device.
\emph{SrcOnly} can be effective if the source domains and the target domain have highly similar distributions.
The \emph{S\&TV} baseline trains on both source and target domain data, but the LOOCV score is computed only from the target domain data.
\item \emph{TrAdaBoost}: Two-stage TrAdaBoost.R2; a boosting method tailored for few-shot regression transfer proposed in \citet{PardoeBoosting2010}.
It is an iterative method with early-stopping \cite{PardoeBoosting2010}, for which we use the leave-one-out cross-validation score on the target domain data as the criterion.
As suggested in \citet{PardoeBoosting2010}, we set the maximum number of outer loop iterations at \(30\).
The base predictor is the decision tree regressor with the maximum depth \(6\) \cite{HastieElements2009}.
Note that although TrAdaBoost does not have a clarified transfer assumption,
we compare the performance for reference.
\item \emph{IW}: Importance weighted KRR using RuLSIF \cite{YamadaRelative2011}.
The method directly estimates a relative joint density ratio function \(\frac{\psrc(\z)}{\alpha \psrc(\z) + (1 - \alpha) \ptar(\z)}\) for \(\alpha \in [0, 1)\),
where \(\psrc\) is a hypothetical source distribution created by pooling all source domain data.
Following \citet{YamadaRelative2011}, we experiment on \(\alpha \in \{0, 0.5, 0.95\}\) and report the results separately.
The regularization coefficient \(\lambda'\) is selected from \(\lambda' \in 2^{\{-10, \ldots, 10\}}\) using importance-weighted cross-validation \cite{SugiyamaCovariate2007}.
\item \emph{GDM}: Generalized discrepancy minimization \cite{CortesAdaptation2019}.
This method performs instance-weighted training on the source domain data with the weights that minimize the \emph{generalized discrepancy} (via quadratic programming).
We select the hyper-parameters \(\lambda_r\) from \(2^{\{-10, \ldots, 10\}}\) as suggested by \citet{CortesAdaptation2019}.
The selection criterion is the performance of the trained predictor on the target training labels as the method trains on the source domain data and the target unlabeled data.
\todo{この一文は分かりづらいが後回し．}

\item \emph{Copula}: Non-parametric regular-vine copula method \cite{Lopez-pazSemisupervised2012}.
This method presumes using a specific joint density estimator called regular-vine (R-vine) copulas.
Adaptation is realized in two steps: the first step estimates which components of the constructed R-vine model are different by performing two-sample tests based on maximum mean discrepancy \cite{Lopez-pazSemisupervised2012},
and the second step re-estimates the components in which a change is detected using only the target domain data.
\item \emph{LOO} (reference score): Leave-one-out cross-validated error estimate is also calculated for reference.
It is the average prediction error of predicting for a single held-out test point when the predictor is trained on the rest of the whole target domain data including those in the test set for the other algorithms.
\end{itemize}
\paragraph{Evaluation procedure.}
\label{sec:org43e2b80}
The prediction accuracy was measured by the mean squared error (MSE).
For each train-test split, we randomly select one-third (6 points) of the target domain dataset as the training set and use the rest as the test set.
All experiments were repeated 10 times with different train-test splits of target domain data.
\paragraph{Results.}
\label{sec:org92ac66d}
The results are reported in Table~\ref{tbl:experiment:real-1}.
We report the MSE scores normalized by that of \emph{LOO} to facilitate the comparison, similarly to \citet{CortesDomain2014}.
In many of the target domain choices, the naive baselines (\emph{SrcOnly} and \emph{S\&TV}) suffer from negative transfer, i.e., higher average MSE than \emph{TarOnly} (in 12 out of 18 domains).
On the other hand, the proposed method successfully performs better than \emph{TarOnly} or is more resistant to negative transfer than the other compared methods.
The performances of \emph{GDM}, \emph{Copula}, and \emph{IW} are often inferior even compared to the baseline performance of \emph{SrcAndTarValid}.
For \emph{GDM} and \emph{IW}, this can be attributed to the fact that these methods presume the availability of abundant (unlabeled) target domain data, which is unavailable in the current problem setup.
For \emph{Copula}, the performance inferior to the naive baselines is possibly due to the restriction of the predictor model to its accompanied probability model \cite{Lopez-pazSemisupervised2012}.
\emph{TrAdaBoost} works reasonably well for many but not all domains. For some domains, it suffered from negative transfer similarly to others, possibly because of the very small number of training data points. Note that the transfer assumption of TrAdaBoost has not been stated \cite{PardoeBoosting2010}, and it is not understood when the method is reliable.
The domains on which the baselines perform better than the proposed method can be explained by the following two cases: (1) easier domains allow naive baselines to perform well and (2) some domains may have deviated \(f\).
Case (1) implies that estimating \(f\) is unnecessary, hence the proposed method can be suboptimal (more likely for JPN, NLD, NOR, and SWE, where SrcOnly or S\&TV improve upon TrgOnly).
On the other hand, case (2) implies that an approximation error was induced as in Theorem 2 (more likely for IRL and ITA). In this case, others also perform poorly, implying the difficulty of the problem instance.
In either case, in practice, one may well perform cross-validation to fallback to the baselines.
\begin{table*}[!ht]
\centering
\caption{
Results of the real-world data experiments for different choices of the target domain.
The evaluation score is MSE normalized by that of \emph{LOO} (the lower the better).
All experiments were repeated 10 times with different train-test splits of target domain data,
and the average performance is reported with the standard errors in the brackets.
The target column indicates abbreviated country names.
Bold-face indicates the best score (Prop: proposed method, TrAda: \emph{TrAdaBoost}, the numbers in the brackets of IW indicate the value of \(\alpha\)).
The proposed method often improves upon the baseline \emph{TarOnly} or is relatively more resistant to negative transfer, with notable improvements in \emph{DEU}, \emph{GBR}, and \emph{USA}.
}
\label{tbl:experiment:real-1}
\begin{tabular}[t]{|*1{p{8mm}|}|*1{p{8mm}}||*4{p{10mm}|}*1{p{10mm}|}*5{p{10mm}|}}
  \hline
  Target & (LOO) & TarOnly & \textbf{Prop} & SrcOnly & S\&TV & TrAda & GDM & Copula & IW(.0) & IW(.5) & IW(.95)\\
  \hhline{|=||=||=|=|=|=|=|=|=|=|=|=|}
  AUT & 1 & 5.88 (1.60) & \textbf{5.39 (1.86)} & 9.67 (0.57) & 9.84 (0.62) & 5.78 (2.15) & 31.56 (1.39) & 27.33 (0.77) & 39.72 (0.74) & 39.45 (0.72) & 39.18 (0.76)\\
  \hline
BEL & 1 & 10.70 (7.50) & \textbf{7.94 (2.19)} & 8.19 (0.68) & 9.48 (0.91) & 8.10 (1.88) & 89.10 (4.12) & 119.86 (2.64) & 105.15 (2.96) & 105.28 (2.95) & 104.30 (2.95)\\
\hline
CAN & 1 & 5.16 (1.36) & \textbf{3.84 (0.98)} & 157.74 (8.83) & 156.65 (10.69) & 51.94 (30.06) & 516.90 (4.45) & 406.91 (1.59) & 592.21 (1.87) & 591.21 (1.84) & 589.87 (1.91)\\
\hline
DNK & 1 & 3.26 (0.61) & \textbf{3.23 (0.63)} & 30.79 (0.93) & 28.12 (1.67) & 25.60 (13.11) & 16.84 (0.85) & 14.46 (0.79) & 22.15 (1.10) & 22.11 (1.10) & 21.72 (1.07)\\
\hline
FRA & 1 & 2.79 (1.10) & \textbf{1.92 (0.66)} & 4.67 (0.41) & 3.05 (0.11) & 52.65 (25.83) & 91.69 (1.34) & 156.29 (1.96) & 116.32 (1.27) & 116.54 (1.25) & 115.29 (1.28)\\
\hline
DEU & 1 & 16.99 (8.04) & \textbf{6.71 (1.23)} & 229.65 (9.13) & 210.59 (14.99) & 341.03 (157.80) & 739.29 (11.81) & 929.03 (4.85) & 817.50 (4.60) & 818.13 (4.55) & 812.60 (4.57)\\
\hline
GRC & 1 & 3.80 (2.21) & \textbf{3.55 (1.79)} & 5.30 (0.90) & 5.75 (0.68) & 11.78 (2.36) & 26.90 (1.89) & 23.05 (0.53) & 47.07 (1.92) & 45.50 (1.82) & 45.72 (2.00)\\
\hline
IRL & 1 & \textbf{3.05 (0.34)} & 4.35 (1.25) & 135.57 (5.64) & 12.34 (0.58) & 23.40 (17.50) & 3.84 (0.22) & 26.60 (0.59) & 6.38 (0.13) & 6.31 (0.14) & 6.16 (0.13)\\
\hline
ITA & 1 & \textbf{13.00 (4.15)} & 14.05 (4.81) & 35.29 (1.83) & 39.27 (2.52) & 87.34 (24.05) & 226.95 (11.14) & 343.10 (10.04) & 244.25 (8.50) & 244.84 (8.58) & 242.60 (8.46)\\
\hline
JPN & 1 & 10.55 (4.67) & 12.32 (4.95) & \textbf{8.10 (1.05)} & 8.38 (1.07) & 18.81 (4.59) & 95.58 (7.89) & 71.02 (5.08) & 135.24 (13.57) & 134.89 (13.50) & 134.16 (13.43)\\
\hline
NLD & 1 & 3.75 (0.80) & 3.87 (0.79) & \textbf{0.99 (0.06)} & 0.99 (0.05) & 9.45 (1.43) & 28.35 (1.62) & 29.53 (1.58) & 33.28 (1.78) & 33.23 (1.77) & 33.14 (1.77)\\
\hline
NOR & 1 & 2.70 (0.51) & 2.82 (0.73) & 1.86 (0.29) & \textbf{1.63 (0.11)} & 24.25 (12.50) & 23.36 (0.88) & 31.37 (1.17) & 27.86 (0.94) & 27.86 (0.93) & 27.52 (0.91)\\
\hline
ESP & 1 & 5.18 (1.05) & 6.09 (1.53) & 5.17 (1.14) & \textbf{4.29 (0.72)} & 14.85 (4.20) & 33.16 (6.99) & 152.59 (6.19) & 53.53 (2.47) & 52.56 (2.42) & 52.06 (2.40)\\
\hline
SWE & 1 & 6.44 (2.66) & 5.47 (2.63) & 2.48 (0.23) & \textbf{2.02 (0.21)} & 2.18 (0.25) & 15.53 (2.59) & 2706.85 (17.91) & 118.46 (1.64) & 118.23 (1.64) & 118.27 (1.64)\\
\hline
CHE & 1 & 3.51 (0.46) & \textbf{2.90 (0.37)} & 43.59 (1.77) & 7.48 (0.49) & 38.32 (9.03) & 8.43 (0.24) & 29.71 (0.53) & 9.72 (0.29) & 9.71 (0.29) & 9.79 (0.28)\\
\hline
TUR & 1 & 1.65 (0.47) & 1.06 (0.15) & 1.22 (0.18) & \textbf{0.91 (0.09)} & 2.19 (0.34) & 64.26 (5.71) & 142.84 (2.04) & 159.79 (2.63) & 157.89 (2.63) & 157.13 (2.69)\\
\hline
GBR & 1 & 5.95 (1.86) & \textbf{2.66 (0.57)} & 15.92 (1.02) & 10.05 (1.47) & 7.57 (5.10) & 50.04 (1.75) & 68.70 (1.25) & 70.98 (1.01) & 70.87 (0.99) & 69.72 (1.01)\\
\hline
USA & 1 & 4.98 (1.96) & \textbf{1.60 (0.42)} & 21.53 (3.30) & 12.28 (2.52) & 2.06 (0.47) & 308.69 (5.20) & 244.90 (1.82) & 462.51 (2.14) & 464.75 (2.08) & 465.88 (2.16)\\
\hline
\hhline{|=||=||=|=|=|=|=|=|=|=|=|=|}
\#Best & - & 2 & 10 & 2 & 4 & 0 & 0 & 0 & 0 & 0 & 0\\
\hline
\end{tabular}

\end{table*}
\section{Conclusion \label{paper:sec:conclusion}}
\label{sec:orgab1022a}
In this paper, we proposed a novel few-shot supervised DA method for regression problems based on the assumption of shared generative mechanism.
Through theoretical and experimental analysis, we demonstrated the effectiveness of the proposed approach.
By considering the latent common structure behind the domain distributions, the proposed method successfully induces positive transfer even when a naive usage of the source domain data can suffer from negative transfer.
Our future work includes
making an experimental comparison with extensively more datasets and methods
as well as an extension to the case where the underlying mechanism are not exactly identical but similar.
\todotwo{(Don't reveal too much about future work.)}
\section*{Acknowledgments}
The authors would like to thank the anonymous reviewers for their insightful comments and thorough discussions.
We would also like to thank Yuko Kuroki and Taira Tsuchiya for proofreading the manuscript.
This work was supported by RIKEN Junior Research Associate Program.
TT was supported by Masason Foundation.
IS was supported by KAKEN 17H04693.
MS was supported by JST CREST Grant Number JPMJCR18A2.
\bibliography{causal_transfer}

\begin{thebibliography}{65}
\providecommand{\natexlab}[1]{#1}
\providecommand{\url}[1]{\texttt{#1}}
\expandafter\ifx\csname urlstyle\endcsname\relax
  \providecommand{\doi}[1]{doi: #1}\else
  \providecommand{\doi}{doi: \begingroup \urlstyle{rm}\Url}\fi

\bibitem[Qui(2009)]{Quinonero-CandelaDataset2009}
\emph{Dataset {{Shift}} in {{Machine Learning}}}.
\newblock Neural Information Processing Series. {MIT Press}, {Cambridge, Mass},
  2009.

\bibitem[Adams and Fournier(2003)]{AdamsSobolev2003}
Robert~A. Adams and John~JF Fournier.
\newblock \emph{Sobolev {{Spaces}}}.
\newblock {Academic press}, 2003.

\bibitem[Arjovsky et~al.(2020)Arjovsky, Bottou, Gulrajani, and
  {Lopez-Paz}]{ArjovskyInvariant2020}
Martin Arjovsky, L{\'e}on Bottou, Ishaan Gulrajani, and David {Lopez-Paz}.
\newblock Invariant risk minimization.
\newblock \emph{arXiv:1907.02893 [cs, stat]}, March 2020.

\bibitem[Baltagi(2005)]{BaltagiEconometric2005}
Badi Baltagi.
\newblock \emph{Econometric {{Analysis}} of {{Panel Data}}}.
\newblock {New York: John Wiley and Sons}, 3rd edition, 2005.

\bibitem[Baltagi and Griffin(1983)]{BaltagiGasoline1983}
Badi~H. Baltagi and James~M. Griffin.
\newblock Gasoline demand in the {{OECD}}: {{An}} application of pooling and
  testing procedures.
\newblock \emph{European Economic Review}, 22\penalty0 (2):\penalty0 117--137,
  1983.

\bibitem[{Ben-David} et~al.(2007){Ben-David}, Blitzer, Crammer, and
  Pereira]{Ben-DavidAnalysis2007}
Shai {Ben-David}, John Blitzer, Koby Crammer, and Fernando Pereira.
\newblock Analysis of representations for domain adaptation.
\newblock In \emph{Advances in {{Neural Information Processing Systems}} 19},
  pages 137--144. {MIT Press}, 2007.

\bibitem[{Ben-David} et~al.(2010){Ben-David}, Blitzer, Crammer, Kulesza,
  Pereira, and Vaughan]{Ben-Davidtheory2010}
Shai {Ben-David}, John Blitzer, Koby Crammer, Alex Kulesza, Fernando Pereira,
  and Jennifer~Wortman Vaughan.
\newblock A theory of learning from different domains.
\newblock \emph{Machine Learning}, 79\penalty0 (1-2):\penalty0 151--175, 2010.

\bibitem[Blitzer et~al.(2008)Blitzer, Crammer, Kulesza, Pereira, and
  Wortman]{BlitzerLearning2008}
John Blitzer, Koby Crammer, Alex Kulesza, Fernando Pereira, and Jennifer
  Wortman.
\newblock Learning bounds for domain adaptation.
\newblock In \emph{Advances in {{Neural Information Processing Systems}} 20},
  pages 129--136. {Curran Associates, Inc.}, 2008.

\bibitem[Cl{\'e}men{\c c}on et~al.(2016)Cl{\'e}men{\c c}on, Colin, and
  Bellet]{ClemenconScalingup2016}
Stephan Cl{\'e}men{\c c}on, Igor Colin, and Aur{\'e}lien Bellet.
\newblock Scaling-up empirical risk minimization: Optimization of incomplete
  {{U}}-statistics.
\newblock \emph{Journal of Machine Learning Research}, 17\penalty0
  (76):\penalty0 1--36, 2016.

\bibitem[Cortes and Mohri(2014)]{CortesDomain2014}
Corinna Cortes and Mehryar Mohri.
\newblock Domain adaptation and sample bias correction theory and algorithm for
  regression.
\newblock \emph{Theoretical Computer Science}, 519:\penalty0 103--126, 2014.

\bibitem[Cortes et~al.(2019)Cortes, Mohri, and Medina]{CortesAdaptation2019}
Corinna Cortes, Mehryar Mohri, and Andr{\'e}s~Mu{\~n}oz Medina.
\newblock Adaptation based on generalized discrepancy.
\newblock \emph{Journal of Machine Learning Research}, 20\penalty0
  (1):\penalty0 1--30, 2019.

\bibitem[Courty et~al.(2017)Courty, Flamary, Habrard, and
  Rakotomamonjy]{CourtyJoint2017a}
Nicolas Courty, R{\'e}mi Flamary, Amaury Habrard, and Alain Rakotomamonjy.
\newblock Joint distribution optimal transportation for domain adaptation.
\newblock In \emph{Advances in {{Neural Information Processing Systems}} 30},
  pages 3730--3739. {Curran Associates, Inc.}, 2017.

\bibitem[Ghassami et~al.(2017)Ghassami, Salehkaleybar, Kiyavash, and
  Zhang]{GhassamiLearning2017a}
AmirEmad Ghassami, Saber Salehkaleybar, Negar Kiyavash, and Kun Zhang.
\newblock Learning causal structures using regression invariance.
\newblock In \emph{Advances in {{Neural Information Processing Systems}} 30},
  pages 3011--3021. {Curran Associates, Inc.}, 2017.

\bibitem[Golub and Van~Loan(2013)]{GolubMatrix2013}
Gene~H. Golub and Charles~F. Van~Loan.
\newblock \emph{Matrix {{Computations}}}.
\newblock Johns {{Hopkins}} Studies in the Mathematical Sciences. {The Johns
  Hopkins University Press}, {Baltimore}, 4th edition, 2013.

\bibitem[Gong et~al.(2016)Gong, Zhang, Liu, Tao, Glymour, and
  Sch{\"o}lkopf]{GongDomain2016b}
Mingming Gong, Kun Zhang, Tongliang Liu, Dacheng Tao, Clark Glymour, and
  Bernhard Sch{\"o}lkopf.
\newblock Domain adaptation with conditional transferable components.
\newblock In \emph{Proceedings of the 33rd International Conference on Machine
  Learning}, pages 2839--2848, {New York, USA}, 2016. {PMLR}.

\bibitem[Gong et~al.(2018)Gong, Zhang, Huang, Glymour, Tao, and
  Batmanghelich]{GongCausal2018}
Mingming Gong, Kun Zhang, Biwei Huang, Clark Glymour, Dacheng Tao, and Kayhan
  Batmanghelich.
\newblock Causal generative domain adaptation networks.
\newblock \emph{arXiv:1804.04333 [cs, stat]}, April 2018.

\bibitem[Greene(2012)]{GreeneEconometric2012}
William~H. Greene.
\newblock \emph{Econometric {{Analysis}}}.
\newblock {Prentice Hall}, {Boston}, 7th edition, 2012.

\bibitem[Hastie et~al.(2009)Hastie, Tibshirani, and
  Friedman]{HastieElements2009}
Trevor Hastie, Robert Tibshirani, and Jerome Friedman.
\newblock \emph{{The Elements of Statistical Learning: Data Mining, Inference,
  and Prediction}}.
\newblock {Springer}, 2nd edition, 2009.

\bibitem[Hayfield and Racine(2008)]{HayfieldNonparametric2008}
Tristen Hayfield and Jeffrey~S. Racine.
\newblock Nonparametric econometrics: {{The}} np package.
\newblock \emph{Journal of Statistical Software}, 27\penalty0 (5), 2008.

\bibitem[H{\"u}nermund and Bareinboim(2019)]{HunermundCausal2019}
Paul H{\"u}nermund and Elias Bareinboim.
\newblock Causal inference and data-fusion in econometrics.
\newblock \emph{arXiv:1912.09104 [econ]}, December 2019.

\bibitem[Hyv{\"a}rinen and Pajunen(1999)]{HyvarinenNonlinear1999}
A.~Hyv{\"a}rinen and P.~Pajunen.
\newblock Nonlinear independent component analysis: Existence and uniqueness
  results.
\newblock \emph{Neural networks}, 12\penalty0 (3):\penalty0 429--439, 1999.

\bibitem[Hyv{\"a}rinen and Morioka(2016)]{HyvarinenUnsupervised2016}
Aapo Hyv{\"a}rinen and Hiroshi Morioka.
\newblock Unsupervised feature extraction by time-contrastive learning and
  nonlinear {{ICA}}.
\newblock In \emph{Advances in {{Neural Information Processing Systems}} 29},
  pages 3765--3773. {Curran Associates, Inc.}, 2016.

\bibitem[Hyv{\"a}rinen and Morioka(2017)]{HyvarinenNonlinear2017}
Aapo Hyv{\"a}rinen and Hiroshi Morioka.
\newblock Nonlinear {{ICA}} of temporally dependent stationary sources.
\newblock In \emph{Proceedings of the 20th {{International Conference}} on
  {{Artificial Intelligence}} and {{Statistics}}}, pages 460--469, 2017.

\bibitem[Hyv{\"a}rinen et~al.(2019)Hyv{\"a}rinen, Sasaki, and
  Turner]{HyvarinenNonlinear2019}
Aapo Hyv{\"a}rinen, Hiroaki Sasaki, and Richard Turner.
\newblock Nonlinear {{ICA}} using auxiliary variables and generalized
  contrastive learning.
\newblock In \emph{Proceedings of the 22nd {{International Conference}} on
  {{Artificial Intelligence}} and {{Statistics}}}, pages 859--868, 2019.

\bibitem[Ipsen and Rehman(2008)]{IpsenPerturbation2008}
Ilse C.~F. Ipsen and Rizwana Rehman.
\newblock Perturbation bounds for determinants and characteristic polynomials.
\newblock \emph{SIAM Journal on Matrix Analysis and Applications}, 30\penalty0
  (2):\penalty0 762--776, 2008.

\bibitem[Kano and Shimizu(2003)]{KanoCausal2003}
Yutaka Kano and Shohei Shimizu.
\newblock Causal inference using nonnormality.
\newblock In \emph{Proceedings of the {{International Symposium}} on the
  {{Science}} of {{Modeling}}, the 30th {{Anniversary}} of the {{Information
  Criterion}},}, pages 261--270, 2003.

\bibitem[Khemakhem et~al.(2019)Khemakhem, Kingma, Monti, and
  Hyv{\"a}rinen]{KhemakhemVariational2019}
Ilyes Khemakhem, Diederik~P. Kingma, Ricardo~Pio Monti, and Aapo Hyv{\"a}rinen.
\newblock Variational autoencoders and nonlinear {{ICA}}: A unifying framework.
\newblock \emph{arXiv:1907.04809 [cs, stat]}, July 2019.

\bibitem[Kingma and Ba(2017)]{KingmaAdam2017}
Diederik~P. Kingma and Jimmy Ba.
\newblock Adam: A method for stochastic optimization.
\newblock \emph{arXiv:1412.6980 [cs]}, January 2017.

\bibitem[Kingma and Dhariwal(2018)]{KingmaGlow2018}
Durk~P Kingma and Prafulla Dhariwal.
\newblock Glow: Generative flow with invertible 1x1 convolutions.
\newblock In \emph{Advances in {{Neural Information Processing Systems}} 31},
  pages 10215--10224. {Curran Associates, Inc.}, 2018.

\bibitem[Kumagai(2016)]{KumagaiLearning2016a}
Wataru Kumagai.
\newblock Learning bound for parameter transfer learning.
\newblock In \emph{Advances in {{Neural Information Processing Systems}} 29},
  pages 2721--2729. {Curran Associates, Inc.}, 2016.

\bibitem[Kuroki et~al.(2019)Kuroki, Charoenphakdee, Bao, Honda, Sato, and
  Sugiyama]{KurokiUnsupervised2019}
Seiichi Kuroki, Nontawat Charoenphakdee, Han Bao, Junya Honda, Issei Sato, and
  Masashi Sugiyama.
\newblock Unsupervised domain adaptation based on source-guided discrepancy.
\newblock In \emph{Proceedings of the {{AAAI Conference}} on {{Artificial
  Intelligence}}}, volume~33, pages 4122--4129, 2019.

\bibitem[LeCun et~al.(2015)LeCun, Bengio, and Hinton]{LeCunDeep2015}
Yann LeCun, Yoshua Bengio, and Geoffrey Hinton.
\newblock Deep learning.
\newblock \emph{Nature}, 521\penalty0 (7553):\penalty0 436--444, 2015.

\bibitem[Lee(1990)]{LeeUStatistics1990}
A.~J. Lee.
\newblock \emph{U-{{Statistics}}: {{Theory}} and {{Practice}}}.
\newblock {M. Dekker}, {New York}, 1990.

\bibitem[Lee et~al.(2009)Lee, Raina, Teichman, and Ng]{LeeExponential2009}
Honglak Lee, Rajat Raina, Alex Teichman, and Andrew~Y. Ng.
\newblock Exponential family sparse coding with applications to self-taught
  learning.
\newblock In \emph{Proceedings of the 21st {{International Jont Conference}} on
  {{Artifical Intelligence}}}, pages 1113--1119, {San Francisco, CA, USA},
  2009. {Morgan Kaufmann Publishers Inc.}

\bibitem[{Lopez-paz} et~al.(2012){Lopez-paz}, {Hern{\'a}ndez-lobato}, and
  Sch{\"o}lkopf]{Lopez-pazSemisupervised2012}
David {Lopez-paz}, Jose~M. {Hern{\'a}ndez-lobato}, and Bernhard Sch{\"o}lkopf.
\newblock Semi-supervised domain adaptation with non-parametric copulas.
\newblock In \emph{Advances in {{Neural Information Processing Systems}} 25},
  pages 665--673. {Curran Associates, Inc.}, 2012.

\bibitem[Magliacane et~al.(2018)Magliacane, {van Ommen}, Claassen, Bongers,
  Versteeg, and Mooij]{MagliacaneDomain2018}
Sara Magliacane, Thijs {van Ommen}, Tom Claassen, Stephan Bongers, Philip
  Versteeg, and Joris~M Mooij.
\newblock Domain adaptation by using causal inference to predict invariant
  conditional distributions.
\newblock In \emph{Advances in {{Neural Information Processing Systems}} 31},
  pages 10846--10856. {Curran Associates, Inc.}, 2018.

\bibitem[Mohri et~al.(2012)Mohri, Rostamizadeh, and
  Talwalkar]{MohriFoundations2012}
Mehryar Mohri, Afshin Rostamizadeh, and Ameet Talwalkar.
\newblock \emph{Foundations of {{Machine Learning}}}.
\newblock Adaptive Computation and Machine Learning Series. {MIT Press},
  {Cambridge, MA}, 2012.

\bibitem[Monti et~al.(2019)Monti, Zhang, and Hyv{\"a}rinen]{MontiCausal2019b}
Ricardo~Pio Monti, Kun Zhang, and Aapo Hyv{\"a}rinen.
\newblock Causal discovery with general non-linear relationships using
  non-linear {{ICA}}.
\newblock In \emph{Proceedings of the {{Thirty}}-{{Fifth Conference}} on
  {{Uncertainty}} in {{Artificial Intelligence}}}, 2019.

\bibitem[Ng(1997)]{NgPreventing1997}
Andrew~Y. Ng.
\newblock Preventing ``overfitting'' of cross-validation data.
\newblock In \emph{Proceedings of the {{Fourteenth International Conference}}
  on {{Machine Learning}}}, pages 245--253, {San Francisco, CA, USA}, 1997.

\bibitem[Nguyen et~al.(2016)Nguyen, Christoffel, and
  Sugiyama]{NguyenContinuous2016}
Tuan~Duong Nguyen, Marthinus Christoffel, and Masashi Sugiyama.
\newblock Continuous {{Target Shift Adaptation}} in {{Supervised Learning}}.
\newblock In \emph{Asian {{Conference}} on {{Machine Learning}}}, volume~45 of
  \emph{Proceedings of {{Machine Learning Research}}}, pages 285--300. {PMLR},
  2016.

\bibitem[Pan et~al.(2011)Pan, Tsang, Kwok, and Yang]{PanDomain2011}
Sinno~Jialin Pan, Ivor~W. Tsang, James~T. Kwok, and Qiang Yang.
\newblock Domain adaptation via transfer component analysis.
\newblock \emph{IEEE Transactions on Neural Networks}, 22\penalty0
  (2):\penalty0 199--210, 2011.

\bibitem[Papa et~al.(2015)Papa, Cl{\'e}men{\c c}on, and Bellet]{PapaSGD2015}
Guillaume Papa, St{\'e}phan Cl{\'e}men{\c c}on, and Aur{\'e}lien Bellet.
\newblock {{SGD Algorithms}} based on {{Incomplete U}}-statistics:
  {{Large}}-{{Scale Minimization}} of {{Empirical Risk}}.
\newblock In \emph{Advances in {{Neural Information Processing Systems}} 28},
  pages 1027--1035. {Curran Associates, Inc.}, 2015.

\bibitem[Pardoe and Stone(2010)]{PardoeBoosting2010}
David Pardoe and Peter Stone.
\newblock Boosting for regression transfer.
\newblock In \emph{Proceedings of the {{Twenty}}-{{Seventh International
  Conference}} on {{Machine Learning}}}, pages 863--870, {Haifa, Israel}, 2010.

\bibitem[Paszke et~al.(2019)Paszke, Gross, Massa, Lerer, Bradbury, Chanan,
  Killeen, Lin, Gimelshein, Antiga, Desmaison, Kopf, Yang, DeVito, Raison,
  Tejani, Chilamkurthy, Steiner, Fang, Bai, and Chintala]{PaszkePyTorch2019}
Adam Paszke, Sam Gross, Francisco Massa, Adam Lerer, James Bradbury, Gregory
  Chanan, Trevor Killeen, Zeming Lin, Natalia Gimelshein, Luca Antiga, Alban
  Desmaison, Andreas Kopf, Edward Yang, Zachary DeVito, Martin Raison, Alykhan
  Tejani, Sasank Chilamkurthy, Benoit Steiner, Lu~Fang, Junjie Bai, and Soumith
  Chintala.
\newblock {{PyTorch}}: {{An}} imperative style, high-performance deep learning
  library.
\newblock In \emph{Advances in {{Neural Information Processing Systems}} 32},
  pages 8024--8035. {Curran Associates, Inc.}, 2019.

\bibitem[Pearl(2009)]{PearlCausality2009}
Judea Pearl.
\newblock \emph{Causality: {{Models}}, {{Reasoning}} and {{Inference}}}.
\newblock {Cambridge University Press}, {Cambridge, U.K. ; New York}, 2nd
  edition, 2009.

\bibitem[Peters et~al.(2017)Peters, Janzing, and
  Sch{\"o}lkopf]{PetersElements2017}
Jonas Peters, Dominik Janzing, and Bernhard Sch{\"o}lkopf.
\newblock \emph{Elements of {{Causal Inference}}: {{Foundations}} and
  {{Learning Algorithms}}}.
\newblock Adaptive Computation and Machine Learning Series. {The MIT Press},
  {Cambridge, Massachuestts}, 2017.

\bibitem[{R Core Team}(2018)]{RCoreTeamlanguage2018}
{R Core Team}.
\newblock \emph{R: {{A}} Language and Environment for Statistical Computing}.
\newblock {Vienna, Austria}, 2018.

\bibitem[Reiss and Wolak(2007)]{ReissStructural2007}
Peter~C. Reiss and Frank~A. Wolak.
\newblock Structural econometric modeling: Rationales and examples from
  industrial organization.
\newblock In \emph{Handbook of {{Econometrics}}}, volume~6, pages 4277--4415.
  {Elsevier}, 2007.

\bibitem[Rejchel(2012)]{Rejchelranking2012}
Wojciech Rejchel.
\newblock On ranking and generalization bounds.
\newblock \emph{Journal of Machine Learning Research}, 13\penalty0
  (May):\penalty0 1373--1392, 2012.

\bibitem[{Rojas-Carulla} et~al.(2018){Rojas-Carulla}, Sch{\"o}lkopf, Turner,
  and Peters]{Rojas-CarullaInvariant2018}
Mateo {Rojas-Carulla}, Bernhard Sch{\"o}lkopf, Richard Turner, and Jonas
  Peters.
\newblock Invariant models for causal transfer learning.
\newblock \emph{Journal of Machine Learning Research}, 19\penalty0
  (36):\penalty0 1--34, 2018.

\bibitem[Sch{\"o}lkopf et~al.(2000)Sch{\"o}lkopf, Williamson, Smola,
  {Shawe-Taylor}, and Platt]{ScholkopfSupport2000}
Bernhard Sch{\"o}lkopf, Robert~C Williamson, Alex~J. Smola, John
  {Shawe-Taylor}, and John~C. Platt.
\newblock Support vector method for novelty detection.
\newblock In \emph{Advances in {{Neural Information Processing Systems}} 12},
  pages 582--588. {MIT Press}, 2000.

\bibitem[Sch{\"o}lkopf et~al.(2012)Sch{\"o}lkopf, Janzing, Peters, Sgouritsa,
  Zhang, and Mooij]{Scholkopfcausal2012a}
Bernhard Sch{\"o}lkopf, Dominik Janzing, Jonas Peters, Eleni Sgouritsa, Kun
  Zhang, and Joris Mooij.
\newblock On causal and anticausal learning.
\newblock In \emph{Proceedings of the 29th {{International Coference}} on
  {{Machine Learning}}}, pages 459--466. {Omnipress}, 2012.

\bibitem[Sherman(1994)]{ShermanMaximal1994}
Robert~P. Sherman.
\newblock Maximal inequalities for degenerate {{U}}-processes with applications
  to optimization estimators.
\newblock \emph{The Annals of Statistics}, 22\penalty0 (1):\penalty0 439--459,
  1994.

\bibitem[Shimizu et~al.(2006)Shimizu, Hoyer, Hyv{\"a}rinen, and
  Kerminen]{Shimizulinear2006}
Shohei Shimizu, Patrik~O Hoyer, Aapo Hyv{\"a}rinen, and Antti~J Kerminen.
\newblock A linear non-{{Gaussian}} acyclic model for causal discovery.
\newblock \emph{Journal of Machine Learning Research}, 7\penalty0
  (October):\penalty0 2003--2030, 2006.

\bibitem[Shimodaira(2000)]{ShimodairaImproving2000}
Hidetoshi Shimodaira.
\newblock Improving predictive inference under covariate shift by weighting the
  log-likelihood function.
\newblock \emph{Journal of Statistical Planning and Inference}, 90\penalty0
  (2):\penalty0 227--244, 2000.

\bibitem[Stojanov et~al.(2019)Stojanov, Gong, Carbonell, and
  Zhang]{StojanovDatadriven2019}
Petar Stojanov, Mingming Gong, Jaime Carbonell, and Kun Zhang.
\newblock Data-driven approach to multiple-source domain adaptation.
\newblock In \emph{Proceedings of Machine Learning Research}, volume~89, pages
  3487--3496. {PMLR}, 2019.

\bibitem[Storkey and Sugiyama(2007)]{StorkeyMixture2007}
Amos~J Storkey and Masashi Sugiyama.
\newblock Mixture regression for covariate shift.
\newblock In \emph{Advances in Neural Information Processing Systems 19}, pages
  1337--1344. {MIT Press}, 2007.

\bibitem[Sugiyama et~al.(2007)Sugiyama, Krauledat, and
  M{\"u}ller]{SugiyamaCovariate2007}
Masashi Sugiyama, Matthias Krauledat, and Klaus-Robert M{\"u}ller.
\newblock Covariate shift adaptation by importance weighted cross validation.
\newblock \emph{Journal of Machine Learning Research}, 8\penalty0
  (May):\penalty0 985--1005, 2007.

\bibitem[Wainwright(2019)]{WainwrightHighDimensional2019}
Martin~J. Wainwright.
\newblock \emph{High-{{Dimensional Statistics}}: {{A Non}}-{{Asymptotic
  Viewpoint}}}.
\newblock {Cambridge University Press}, 1st edition, 2019.

\bibitem[Xu et~al.(2014)Xu, Fan, Wu, Chen, Guo, Zhang, and
  Yao]{XupoolingLiNGAM2014}
Lele Xu, Tingting Fan, Xia Wu, KeWei Chen, Xiaojuan Guo, Jiacai Zhang, and
  Li~Yao.
\newblock A pooling-{{LiNGAM}} algorithm for effective connectivity analysis of
  {{fMRI}} data.
\newblock \emph{Frontiers in Computational Neuroscience}, 8\penalty0
  (October):\penalty0 125, 2014.

\bibitem[Yadav et~al.(2018)Yadav, Steinbach, Kumar, and Simon]{YadavMining2018}
Pranjul Yadav, Michael Steinbach, Vipin Kumar, and Gyorgy Simon.
\newblock Mining electronic health records ({{EHRs}}): A survey.
\newblock \emph{ACM Computing Surveys}, 50\penalty0 (6):\penalty0 1--40, 2018.

\bibitem[Yamada et~al.(2011)Yamada, Suzuki, Kanamori, Hachiya, and
  Sugiyama]{YamadaRelative2011}
Makoto Yamada, Taiji Suzuki, Takafumi Kanamori, Hirotaka Hachiya, and Masashi
  Sugiyama.
\newblock Relative density-ratio estimation for robust distribution comparison.
\newblock In \emph{Advances in {{Neural Information Processing Systems}} 24},
  pages 594--602. {Curran Associates, Inc.}, 2011.

\bibitem[Zhang et~al.(2015)Zhang, Gong, and
  Sch{\"o}lkopf]{ZhangMultisource2015}
K.~Zhang, M.~Gong, and B.~Sch{\"o}lkopf.
\newblock Multi-source domain adaptation: A causal view.
\newblock In \emph{Proceedings of the Twenty-Ninth {{AAAI}} Conference on
  Artificial Intelligence}, pages 3150--3157. {AAAI Press}, 2015.

\bibitem[Zhang et~al.(2013)Zhang, Sch{\"o}lkopf, Muandet, and
  Wang]{ZhangDomain2013}
Kun Zhang, Bernhard Sch{\"o}lkopf, Krikamol Muandet, and Zhikun Wang.
\newblock Domain adaptation under target and conditional shift.
\newblock In \emph{Proceedings of the 30th {{International Conference}} on
  {{Machine Learning}}}, pages 819--827, 2013.

\bibitem[Zhang et~al.(2019)Zhang, Liu, Long, and Jordan]{ZhangBridging2019}
Yuchen Zhang, Tianle Liu, Mingsheng Long, and Michael Jordan.
\newblock Bridging theory and algorithm for domain adaptation.
\newblock In \emph{Proceedings of the 36th {{International Conference}} on
  {{Machine Learning}}}, pages 7404--7413, {Long Beach, California, USA}, 2019.
  {PMLR}.

\end{thebibliography}
\clearpage

\onecolumn

\global\csname @topnum\endcsname 0
\global\csname @botnum\endcsname 0

\begin{appendices}

This is the \Supplementary\ for ``Few-shot Domain Adaptation by Causal Mechanism Transfer.''
Table~\ref{tbl:notation} summarizes the abbreviations and the symbols used in the paper.
\begin{table}[htbp]
\centering
\begin{center}
\captionof{table}{\label{tbl:notation}List of abbreviations and symbols used in the paper.}
\begin{tabular}{lll}
\hline
Abbreviation / Symbol & Meaning\\
\hline
DA & Domain adaptation\\
TA & Transfer assumption\\
SEM & Structural equation model\\
GCM & Graphical causal model\\
SCM & Structural causal model\\
IC & Independent component\\
ICA & Independent component analysis\\
GCL & Generalized contrastive learning\\
\IIDtext & Independent and identically distributed\\
\hline
\([N]\) & \(\{1, 2, \ldots, N\}\) where \(N \in \mathbb{N}\)\\
\(\Sobolev{k}{p}{\cdot}\) & The \((k, p)\)-Sobolev norm\\
\hline
\(X\) & The predictor random vector (\(\Re^{D-1}\)-valued)\\
\(Y\) & The predicted random variable (\(\Re\)-valued)\\
\(\Z = (X, Y)\) & The joint random variable (\(\ReD\)-valued)\\
\(\S\) & The independent component vector (\(\ReD\)-valued)\\
\hline
\(\Xsp \subset \Re^{D-1}\) & The space of \(X\)\\
\(\Ysp \subset \Re\) & The space of \(Y\)\\
\(\Zsp \subset \ReD\) & The space of \(\Z = (X, Y)\)\\
\(\G \subset \{\g: \ReDminus \to \Re\}\) & Predictor hypothesis class\\
\(\l: \G \times \Zsp \to [0, \LossUpperBound]\) & Loss function\\
\(\R(\g)\) & Target domain risk \(\Etr \l(\g, \Z)\)\\
\(\gstar \in \G\) & Minimizer of target domain risk\\
\hline
\(\Qsp\) & The set of independent distributions\\
\(\f\) & Ground truth mixing function\\
\(\ptar\) & The target joint distribution\\
\(\psrck\) & The joint distribution of source domain \(k\)\\
\(\qtar \in \Qsp\) & The target independent component (IC) distribution\\
\(\qsrck \in \Qsp\) & The IC distribution of source domain \(k\)\\
\hline
\(D\) & The dimension of \(\Zsp\)\\
\(K\) & The number of source domains\\
\(\ntr\) & The size of the target labeled sample\\
\(\nk\) & The size of the labeled sample from source domain \(k\)\\
\hline
\(\DataTar = \{Z_i\}_{i=1}^{\ntr}\) & Target labeled data set\\
\(\Datak = \{\Zki\}_{i=1}^{\nk}\) & Source labeled data set of source domain \(k\)\\
\(\hR(\g)\) & The ordinary empirical risk estimator\\
\(\cbarhR(\g)\) & The proposed risk estimator (Eq.~\eqref{paper:eq:augmented-erm})\\
\(\hf\) & The estimator of \(\f\)\\
\(\{\psi_d\}_{d=1}^D\) & The penultimate layer functions composed with \(\f\) during GCL\\
\hline
\end{tabular}
\end{center}
\end{table}
\section{Preliminary: Nonlinear ICA \label{paper:sec:appendix:nonlinear-ica-gcl}}
\label{sec:org45dd305}
Here, we use the same notation as the main text.
The recently developed nonlinear ICA provides an algorithm to estimate the mixing function \(\f\).
For the case of nonlinear \(\f\), the impossibility of identification (i.e., consistent estimation) of \(\f\) in the one-sample \IID case had been established more than two decades ago \citep{HyvarinenNonlinear1999}.
However, recently, various conditions have been proposed under which \(\f\) can be identified with the help of auxiliary information \cite{HyvarinenUnsupervised2016,HyvarinenNonlinear2017,HyvarinenNonlinear2019,KhemakhemVariational2019}.

The identification condition that is directly relevant to this paper is that of the generalized contrastive learning (GCL) proposed in \citet{HyvarinenNonlinear2019}.
\citet{HyvarinenNonlinear2019} assumes that an auxiliary variable \(\u_i\) from some measurable set \(\Usp\) is obtained for each data point as \(\{(\z_i, \u_i)\}_{i=1}^n\) and that the ICs \(\S = (\S^{(1)}, \ldots, \S^{(D)})\) are conditionally independent given \(u\):
\begin{equation*}\begin{split}
q(\s | u) = \prod_{d=1}^D \qd(\s^{(d)} | u).
\end{split}\end{equation*}
Under such conditions, GCL estimates \(\f\) by training a classification function
\begin{equation}\label{paper:eq:gcl-model}\begin{split}
\rhf(\z, \u) = \sum_{d=1}^D \psi_d(\hfinv(\z)_d, \u)
\end{split}\end{equation}
parametrized by \(\hf\) and \(\{\psi_d\}_{d=1}^D\) with the logistic loss for classifying
\begin{equation*}\begin{split}
(\z, \u) \text{ vs. } (\z, \tilde \u),
\end{split}\end{equation*}
where \(\tilde \u \in \Usp \setminus \{\u\}\).
The key condition for the identification of \(\f\) is the following.
\begin{assumption}[Assumption of variability; {\citealp[Theorem~1]{HyvarinenNonlinear2019}}]
For any \(\z\), there exist \(2 D + 1\) distinct points in \(\Usp\), denoted by \(\{\u_j\}_{j=0}^{2D}\), such that the set of \((2D)\)-dimensional vectors \(\{w(\z | \u_j) - w(\z | \u_0)\}_{j = 1}^{2D}\)
are linearly independent, where
\begin{equation*}\begin{split}
w(\z | \u) := \left(\frac{\partial \log \qdd{1}(\z_1 | \u)}{\partial \z_1}, \ldots, \frac{\partial \log \qdd{D}(\z_D | \u)}{\partial \z_D}, \frac{\partial^2 \log\qdd{1}(\z_1 | \u)}{\partial \z_1^2}, \ldots, \frac{\partial^2 \log\qdd{D}(\z_D | \u)}{\partial \z_D^2} \right). \\
\end{split}\end{equation*}
\label{paper:assum:assumption-of-variability}
\end{assumption}
Under Assumption~\ref{paper:assum:assumption-of-variability} and some regularity conditions,
Theorem~1 of \citet{HyvarinenNonlinear2019} states that the transformation \(\hf\) in Eq.~\eqref{paper:eq:gcl-model} trained by GCL is a consistent estimator of \(\f\) upto additional dimension-wise invertible transformations.
Note that the assumption is intrinsically difficult to confirm based on data due to the unsupervised nature of the problem setting.
In this paper, we use the source domain index as the auxiliary variable and employ GCL for domain adaptation.
The present version of Assumption~\ref{paper:assum:assumption-of-variability} requires that we have at least \(2 D + 1\) distinct source domains.
Although this condition can be restrictive in high-dimensional data, we conjecture that there is a possibility for this assumption to be made less stringent in the future because the identification condition is only known to be a sufficient condition, not a necessary condition.
However, pursuing a refinement of the identification condition is out of the scope of this paper.
Among the various methods for nonlinear ICA, we chose to use GCL \cite{HyvarinenNonlinear2019} because it can operate under a nonparametric assumption on the IC distributions whereas other nonlinear ICA methods \cite{HyvarinenUnsupervised2016,HyvarinenNonlinear2017,KhemakhemVariational2019} may require parametric assumptions.
\section{Experiment Details \label{paper:sec:appendix:experiment-detail}}
\label{sec:org9d69689}
Here, we describe more implementation details of the experiment.
Our experiment code can be found at \url{https://github.com/takeshi-teshima/few-shot-domain-adaptation-by-causal-mechanism-transfer}.
\subsection{Dataset details \label{paper:sec:appendix:dataset}}
\label{sec:orge1d08da}
\paragraph{Gasoline consumption data.}
\label{sec:org52d5dee}
The data was downloaded from \url{http://bcs.wiley.com/he-bcs/Books?action=resource\&bcsId=4338\&itemId=1118672321\&resourceId=13452}.
\subsection{Model details: Invertible neural networks \label{paper:sec:appendix:glow}}
\label{sec:org48823c2}
Here, we describe the details of the Glow architecture \cite{KingmaGlow2018} used in our experiments.
Glow consists of three types of layers which are invertible \emph{by design}, namely affine coupling layers, \(1 \times 1\) convolution layers, and activation normalization (actnorm) layers.
In our implementation, we use actnorm as the first layer, and each of the subsequent layers consists of a \(1 \times 1\) convolution layer followed by an affine coupling layer.
\paragraph{Affine coupling layers.}
\label{sec:orgb947e1e}
The coefficients \(s\) and \(t\) for affine coupling layers in the notation of \citet{KingmaGlow2018} are parametrized by two one-hidden-layer neural networks whose number of hidden units is the same and the first layer parameter is shared.
The activation functions of the first layer, the second layer of \(s\), and the second layer of \(t\) are the rectified linear unit (ReLU) activation \cite{LeCunDeep2015}, the hyperbolic tangent function, and the linear activation function, respectively.
A standard practice of affine coupling layers is to compose the coefficient \(s\) with an exponential function \(x \mapsto \exp(x)\) so as to simplify the computation of the log-determinant of the Jacobian \cite{KingmaGlow2018}.
In our implementation, since we do not require the computation of the log-determinant, we omit this device and instead compose \(x \mapsto (x + 1)\).
The addition of \(1\) shifts the parameter space so that \((s, t) = (0, 0)\) corresponds to the the identity map, where \(0\) denotes the constant zero function.
The split of the affine coupling layers is fixed at \((\lfloor \frac{D}{2} \rfloor, D - \lfloor \frac{D}{2} \rfloor)\).
\paragraph{\(1 \times 1\) convolution layers.}
\label{sec:org3482de4}
We initialize the parameters of the neural networks by \(\mathcal{N}(0, \frac{1}{m})\) where \(m\) is the number of parameters of each layer and \(\mathcal{N}\) is the normal distribution.
\subsection{Model details: Penultimate layer networks}
\label{sec:orgc20cd9a}
We initialize the parameter for each layer of \(\psi_d\) by \(\Unif(-\sqrt{\frac{1}{m}}, \sqrt{\frac{1}{m}})\), where \(m\) is the number of input features and \(\Unif\) is the uniform distribution.
\subsection{Training details}
\label{sec:org8ed0960}
During the training of GCL, we fix the batch size at 32.
\subsection{Compared methods details \label{paper:sec:appendix:compared-methods-details}}
\label{sec:org545b61c}
Here, we detail the methods compared through the experiment.
Note that the present paper focuses on regression problems as our approach is based on ICA, hence the methods for classification domain adaptation are not comparable.
\paragraph{\emph{TrAdaBoost}.}
\label{sec:orga760669}
As suggested in \citet{PardoeBoosting2010}, we use the linear loss function and set the maximum number of internal boosting iterations at \(30\).
\paragraph{\emph{GDM}.}
\label{sec:org86e234c}
We fix the number of sampling required for approximating the maximization in the generalization discrepancy at \(200\).
This method presumes using hypothesis classes in a reproducing kernel Hilbert space (RKHS).
\paragraph{\emph{Copula}.}
\label{sec:orgebf0152}
For this model, the probabilistic model of non-parametric R-vine copula of depth \(1\) is used following \citet{Lopez-pazSemisupervised2012}.
Kernel density estimators with RBF kernel are used for estimating the marginal distributions and the copulas.
The bandwidths of the RBF kernels are determined using the rule-of-thumb implemented as ``normal-reference'' in the \emph{np} package of \emph{R} language \cite{HayfieldNonparametric2008}.
The predictions are made by numerically aggregating the estimated conditional distribution over the interval \([\min_i Y_i - 2 \sigma, \max_i Y_i + 2 \sigma]\)
where \(\sigma\) denotes the square root of the unbiased variance of \(\{Y_i\}_{i=1}^{\nsrc}\).
The aggregation is performed by discretizing the interval into a grid of \(300\) points.
The level of the two-sample test is fixed at \(0.05\) for all combination of the two-sample tests following the experiment code of \citet{Lopez-pazSemisupervised2012}.
This method is a single-source domain adaptation method and we pool all source domain data for adaptation.
\section{Details and Proofs of Theorem~\ref{paper:thm:2} \label{paper:sec:appendix:theory-2}}
\label{sec:orgd7e52e5}
Here, we detail the assumptions, the statement, and the proof of Theorem~\ref{paper:thm:2}.
\subsection{Notation}
\label{sec:org8dcfb9c}
To make the proof self-contained, we first recall some general and problem-specific notation.
In the notation here, we omit the domain identifiers from the distributions and the sample size, such as \emph{Tar} or \emph{Src},
because only the target domain data or their distributions appear in the proofs.
The theorem holds regardless of how \(\hf\) is estimated as long as \(\hf\) is independent of the target domain data.
In the proof, we extend the maximal discrepancy bound of U-statistics previously proved for the case of degree-\(2\) in \citet{Rejchelranking2012}, to allow higher degrees.

\paragraph{General mathematical notation.}
\label{sec:orgb509af4}
We denote the set of natural numbers (resp. real numbers) by \(\Na\) (resp. \(\Re\)).
For any \(N \in \Na\), we define \([N] := \{1, 2, \ldots, N\}\).
We use \(\comb{a}{b}\) to denote the number of \(b\)-combinations of \(a\) elements.
For a finite set \(A\), the notation \(\Mean{a \in A}\) denotes the operator to take an average over \(A\), i.e., \(\Mean{a \in A} h = \frac{1}{|A|}\sum_{a \in A} h(a)\).
For a \(D\)-dimensional function \(h\), we denote its \(j\)-th dimension (\(j \in [D]\)) by suffixing \(h_j\).
For a vector \(\e\), we denote its \(j\)-th element by \(\e^{(j)}\).
We denote the Jacobian determinant of a differentiable function \(\psi\) at \(a\) by \(\Jacobian \psi(a) := \det \frac{\mathrm{d} \psi(a)}{\mathrm{d}a}\).
We denote the identity matrix by \(I\) regardless of the size of the matrices when there is no ambiguity.
For finite dimensional vectors, we denote the \(2\)-norm by \(\ltwo{\cdot}\) and the \(1\)-norm by \(\lone{\cdot}\).
For square matrices, we denote the operator-\(2\) norm by \(\op{\cdot}\) and the operator-\(1\) norm by \(\opone{\cdot}\).
We use \(\SobolevSp{k}{p}\) to denote the Sobolev space (on \(\ReD\)) of order \(k\) and define its associated norm by \(\Sobolev{k}{p}{h} := \left(\sum_{|\alpha|\leq k} \Lp{h^{(\alpha)}}^p\right)^{1/p}\) where \(\alpha\) is a multi-index and \(h^{(\alpha)}\) denotes the partial derivative \(\frac{\partial^{|\alpha|}h}{\partial s_1^{\alpha_1} \cdots \partial s_D^{\alpha_D}}\) \citep[Paragraph~3.1]{AdamsSobolev2003}.
We let \(\DSymmetric\) be the degree-\(D\) symmetric group, \(\jGroupSplit := \jGroupSplitDef\) be the set of \(j\) grouping of indices in \([D]\), and \(\jnCombset := \jnCombsetDef\) be the set of all size-\(j\) combinations (without replacement) of indices in \([n]\).
\paragraph{Distributions and expectations.}
\label{sec:org97ae06d}
We denote by \(\Qsp\) the set of all factorized distributions on \(\mathbb{R}^D\) with absolutely continuous marginals.
For a measure \(P\), we denote its \(j\)-product measure by \(P^j := P \otimes \cdots \otimes P\) (repeated \(j\) times).
We assume that all measures appearing in this proof are absolutely continuous with respect to the Lebesgue measure.
The push-forward of a distribution \(\p\) by a function \(h\) is denoted by \(\pushForward{h}{\p}\).
The expectation of a function \(h\) with respect to measure \(P\) is denoted by \(P h\) (if it exists) by abuse of notation.
We also abuse the notation to use \(\psi(s, P, \ldots, P)\) as the shorthand for \(P^{D-1} \psi(s, S_2', \ldots, S_D')\) where \(\{S_d'\}_{d=2}^D \iid P\).
\subsection{Problem setup}
\label{sec:org781c19a}
We denote the target domain distribution by \(\p\).

We fix a hypothesis class \(\G (\subset \{\g: \ReDminus \to \Re\})\), and our goal is to find a \(\g \in \G\) such that the risk functional
\begin{equation*}\begin{split}
\R(\g) := \int \p(\x) \l(\g, \x) \dx
\end{split}\end{equation*}
is small, where \(\l : \G \times \ReD \to \RePos\) is a loss function.
We denote by \(\gstar\) a minimizer of \(\R\) (assuming it exists).
To this end, we are given the training data \(\Data := \{\Z_i\}_{i=1}^n \iid \p\).
Throughout, we assume \(n \geq D\).
To complement the smallness of \(n\), we assume the existence of a generative mechanism.
Concretely, we assume that there exists a diffeomorphism \(\f: \ReD \to \ReD\) such that \(\q := \pushForward{(\finv)}{\p}\) satisfies \(\q \in \Qsp\).
With this transform, the original risk functional is also expressed as
\begin{equation*}\begin{split}
R(\g) = \int \q(\e) \l(\g, \f(\e)) \de.
\end{split}\end{equation*}
As an estimator of \(\f\), we are given another diffeomorphism \(\hf: \ReD \to \ReD\) such that \(\hf \simeq \f\).
With this \(\hf\), the proposed method converts the dataset \(\Data\) by \(\S_i := \hf(\Z_i)\).
We can regard \(\cData := \{\S_i\}_{i=1}^n \iid \cq\), where \(\cq := \pushForward{(\hfinv \circ \f)}{\q}\).
We use \(\Q\) (resp.\ \(\cQ\)) to denote the probability measure corresponding to the density \(\q\) (resp.\ \(\cq\)).
This conversion results in the relation:
\begin{equation*}\begin{split}
\cq(\s) = \q(\finvhf(\s)) \J{\finvhf}{\s}.
\end{split}\end{equation*}
As a candidate hypothesis \(\g \in \G\), the proposed method selects a minimizer \(\cbarhg \in \G\) of the proposed risk estimator \(\cbarhR\) defined as
\begin{equation}\label{eq:theory:proposed-risk-estimator-1}\begin{split}
\cbarhR(\g) := \frac{1}{n^D} \sum_{(i_1, \ldots, i_D) \in [n]^D} \l(\g, \hf(\hat \s_{i_1}^{(1)}, \ldots, \hat \s_{i_D}^{(D)})).
\end{split}\end{equation}
In the proof, we evaluate its concentration around the expectation \(\cbarR(\g) := \cEData \cbarhR(\g)\).
We use \(\cEData\) to denote the expectation with respect to \(\cData\).
Let \(\cbarg\) denote a hypothesis which minimizes \(\cbarR(\g)\) (assuming it exists).

In what follows, for notational simplicity, we define the \(D\)-variate symmetric function \(\ke\) as
\begin{equation*}\begin{split}
\ke(\e_1, \ldots, \e_D) = \DSymmetricMean \l(\g, \hf(\e_{\pi(1)}^{(1)}, \ldots, \e_{\pi(D)}^{(D)})),
\end{split}\end{equation*}
where \(\DSymmetricMean\) indicates an averaging operation over all permutations (without replacement) of \([D]\).
We use \(\hEData\) to denote the sample average operator with respect to \(\Data\) or \(\cData\), depending on the context.
\subsection{Assumptions}
\label{sec:orga6c2018}
\begin{assumption}[The underlying density function is bounded and Lipschitz continuous]
Assume
\begin{equation*}\begin{split}
\qUpperBound &:= \supeFull \q(\e) < \infty, \quad \qLipschitzConst:= \sup_{\e_1 \neq \e_2}\frac{|\q(\e_1)- \q(\e_2)|}{\|\e_1 - \e_2\|} < \infty.
\end{split}\end{equation*}
\label{assum:q-bdd-and-lipschitz}
\end{assumption}

\begin{assumption}[\(\finv\) is Lipschitz continuous and \Holder continuous]
We assume \(\finv \in \finvHolderClass\) where \(\finvHolderClass\) is the \((1, 1)\)-\Holder space \citep[Paragraph~1.29]{AdamsSobolev2003} and
\begin{equation*}\begin{split}
\finvLipschitzConst &:= \supDiffZ \frac{\|\finv(\z_1) - \finv(\z_2)\|}{\|\z_1 - \z_2\|} < \infty.
\end{split}\end{equation*}
\label{assum:f-lipschitz-and-holder}
\end{assumption}

\begin{assumption}[Bounded derivatives of \(\f\) and \(\finv\)]
Assume that
\begin{equation*}\begin{split}
\dfSupNormConst := \supeFull \msupnrm{\df{\e}(\e)} < \infty, \quad \dfinvSupNormConst := \supzFull \msupnrm{\dfinv{\z}(\z)} < \infty.
\end{split}\end{equation*}
where \(\msupnrm{\cdot}\) denotes the maximum absolute value of the elements of a matrix.
\label{assum:dfinv-operator-bounded}
\end{assumption}

\begin{assumption}[Loss function is bounded and uniformly Lipschitz continuous in \(\Z\)]
The considered loss function takes values in a bounded interval:
\begin{equation*}\begin{split}
\l : \G \times \Zsp \to \LossValSp,
\end{split}\end{equation*}
where \(0 < \LossUpperBound < \infty\).
Also assume
\begin{equation*}\begin{split}
\lossUniformLipschitzConst := \supg \supDiffZ \frac{|\l(\g, \z_1) - \l(\g, \z_2)|}{\|\z_1 - \z_2\|} < \infty.
\end{split}\end{equation*}
\label{assum:loss-uniform-lipschitz}
\end{assumption}

\begin{assumption}[Estimated feature extractor]
Assume \(\hf\) is independent of \(\Data\)
and that \(\f_j - \hf_j \in \SobolevSp{1}{d}\) for all \((j, d) \in [D] \times [D]\).
\label{assum:hf-is-independently-trained}
\end{assumption}
Although \(\hf\) and \(f\) are assumed to be diffeomorphisms in the classical sense (implying that they are strongly differentiable), we introduce the Sobolev space because we want to measure their difference and their difference of derivatives in terms of integration.
\begin{assumption}[Entropic condition: Euclidean class {\citep{ShermanMaximal1994}}]
The function class \(\keSp := \{\ke : \g \in \G\}\) is Euclidean for the envelope \(F\) and constants \(A\) and \(V\) \citep{ShermanMaximal1994}, i.e.,
if \(\mu\) is a measure for which \(\mu F^2 < \infty\), then
\begin{equation*}\begin{split}
D(t, d_\mu, \keSp) \leq A t^{-V}, \quad 0 < t \leq 1,
\end{split}\end{equation*}
where \(d_\mu\) is the pseudo metric defined by
\begin{equation*}\begin{split}
d_\mu(\phi_1, \phi_2) := \left[\mu |\phi_1 - \phi_2|^2 / \mu F^2 \right]^{1/2}
\end{split}\end{equation*}
for \(\phi_1, \phi_2 \in \Phi\),
and \(D(t, d_\mu, \keSp)\) denotes the packing number of \(\keSp\) with respect to the pseudometric \(d_\mu\) and radius \(t\).
Without loss of generality, we take the envelope \(F\) such that \(F(\cdot) \leq \LossUpperBound\).
\label{assum:euclidean-class}
\end{assumption}
\begin{assumption}[]
The hypothesis class \(\G\) is expressive enough so that the model approximation error does not expand due to \(\hf\), i.e.,
\begin{equation*}\begin{split}
\infg \cbarR(\g) \leq \infg \R(\g)
\end{split}\end{equation*}
\label{assum:2}
\end{assumption}
The following complexity measure of \(\G\), which is a version of Rademacher complexity for our problem setting, is used to state the theorem.

\begin{definition}[Effective Rademacher complexity]
Define
\begin{equation*}\begin{split}
\Rad(\G) := \frac{1}{n} \cEData\Erad\left[\supg \left|\sum_{i=1}^{n} \rad_i \ETwoToD{\S'}[\ke(\S_i, \S_2', \ldots, \S_D')]\right|\right]
\end{split}\end{equation*}
where \(\{\rad_i\}_{i=1}^n\) are independent uniform sign variables and \(\S_2', \ldots, \S_D' \iid \cQ\) are independent of all other random variables.
\label{def:rademacher-complexity}
\end{definition}
We provide the definition of the ordinary Rademacher complexity in Section~\ref{sec:theory:comparison-of-rademacher} and make a comparison of the two complexity measures in terms of how they depend on the input dimensionality.
\subsection{Theorem statement}
\label{sec:org557b102}
Our goal is to prove the following theorem. This is a detailed version of the theorem appearing in the main body of the paper.
\begin{theorem}[Excess risk bound]
Assume Assumptions~\ref{assum:q-bdd-and-lipschitz},
\ref{assum:f-lipschitz-and-holder}, \ref{assum:dfinv-operator-bounded},
\ref{assum:loss-uniform-lipschitz},
\ref{assum:hf-is-independently-trained}, \ref{assum:euclidean-class},
and \ref{assum:2}.

Then for arbitrary \(\delta, \delta' \in (0, 1)\), we have with probability at least \(1 - (\delta + \delta')\),
\begin{equation*}\begin{split}
&\R(\cbarhg) - \R(\gstar) \\
&\leq \finalErrorBound.
\end{split}\end{equation*}
where
\begin{equation*}\begin{split}
&\ApproxErrorUpperBoundConstOne := \ApproxErrorUpperBoundConstOneDef,\\
&\ApproxDensityDifferenceJacobianPieceCoeffOne := \ApproxDensityDifferenceJacobianPieceCoeffOneDef,\\
&\finalErrorBoundSecondHigherOrder = \finalErrorBoundSecondHigherOrderDef,\\
&\ApproxDensityDifferenceBoundRemainder = \ApproxDensityDifferenceBoundRemainderDef.
\end{split}\end{equation*}
and \(\ApproxDensityDifferenceJacobianPieceCoeff (d = 1, \ldots, D)\) are constants determined in Lemma~\ref{lem:approx-density-bound-2}.
\label{thm:generalization-error}
\end{theorem}
\begin{proof}[Proof of Theorem~\ref{thm:generalization-error}]
By adding and subtracting terms, we have
\begin{equation*}\begin{split}
&\R(\cbarhg) - \R(\gstar) = \annot{(\R - \cbarR)(\cbarhg)}{\text{(A) Approximation error}} + \annot{\cbarR(\cbarhg) - \cbarR(\cbarg)}{\text{(B) Pseudo estimation error}} + \annot{\cbarR(\cbarg) - \R(\gstar)}{(C) Additional model misspecification error}.
\end{split}\end{equation*}
Applying Lemma~\ref{lem:approximation-error} to (A), Lemma~\ref{lem:bound-pseudo-generalization-error} to (B), and Assumption~\ref{assum:2} to (C), we obtain the assertion.
\end{proof}

As it can be seen from the proof above, Theorem~\ref{thm:generalization-error} is proved in two parts, each corresponding to the two lemmas below.
The first lemma evaluates the \emph{approximation error} which reflects the fact that we are approximating \(\f\) by \(\hf\).
\begin{lemma}[Approximation error bound]
Given Assumptions~\ref{assum:q-bdd-and-lipschitz},
\ref{assum:f-lipschitz-and-holder}, \ref{assum:dfinv-operator-bounded},
\ref{assum:loss-uniform-lipschitz}, and \ref{assum:hf-is-independently-trained}.
we have
\begin{equation*}\begin{split}
(\R - \cbarR)(\cbarhg) \leq &\ApproxErrorUpperBound
\end{split}\end{equation*}
where \(\ApproxErrorUpperBoundConstOne\) and \(\ApproxDensityDifferenceBoundRemainder\) are
\begin{equation*}\begin{split}
\ApproxErrorUpperBoundConstOne &:= \ApproxErrorUpperBoundConstOneDef,\\
\ApproxDensityDifferenceBoundRemainder &:= \ApproxDensityDifferenceBoundRemainderDef.
\end{split}\end{equation*}
and \(\ApproxDensityDifferenceJacobianPieceCoeff (d = 1, \ldots, D)\) are constants determined in Lemma~\ref{lem:approx-density-bound-2}.
\label{lem:approximation-error}
\end{lemma}
The second lemma evaluates the \emph{pseudo estimation error} which reflects the fact that we rely on a finite sample to approximate the underlying distribution.
\begin{lemma}[Pseudo estimation error bound]
Assume that Assumptions~\ref{assum:q-bdd-and-lipschitz} and \ref{assum:euclidean-class} hold.
Let the Rademacher complexity be defined as Definition~\ref{def:rademacher-complexity}.
Then for any \(\delta, \delta' \in (0, 1)\), we have with probability at least \(1 - (\delta + \delta')\) that
\begin{equation*}\begin{split}
\cbarR(\cbarhg) - \cbarR(\cbarg) \leq \PseudoGeneralizationErrorBoundOne + \annot{\PseudoGeneralizationErrorBoundTwo}{\(\Order(n^{-1})\)}
\end{split}\end{equation*}
where \(\{\PseudoGenWj\}_{j=1}^D\) are universal constants determined in Lemma~\ref{lem:decompose-v-estimator},
and \(\{\PseudoGenCj\}_{j=2}^D\) are constants determined in Lemma~\ref{lem:rademacher-U-degenerate-part}.
Note that \(\PseudoGenWj = \Order(n^{-(D-j)})\) and \(\wD = \wDContent < 1\).
\label{lem:bound-pseudo-generalization-error}
\end{lemma}

In what follows, we first present some basic facts in Section~\ref{sec:theory:prep-v-u-stats} and provide the proofs for the lemmas.
We provide the proof of Lemma~\ref{lem:approximation-error} in Section~\ref{sec:theory:approx-error-bound},
and that of Lemma~\ref{lem:bound-pseudo-generalization-error} in Section~\ref{sec:theory:pseudo-generalization-bound}.
\subsection{V-statistic and U-statistic \label{sec:theory:prep-v-u-stats}}
\label{sec:orgd36a182}
The theoretical analysis is performed by interpreting the proposed risk estimator Eq.~\eqref{eq:theory:proposed-risk-estimator-1} as a \emph{V-statistic} (explained shortly).
The proofs will be based on applying the following facts in order:
\begin{enumerate}
\item V-statistic can be represented as a weighted average of \emph{U-statistics} with degrees from \(1\) to \(D\), and only the degree-\(D\)) term is the leading term.
\item The degree-\(D\) term is again decomposed into a degree-\(1\) U-statistic and a set of \emph{degenerate} U-statistics.
\item The degree-\(1\) \emph{U-statistic} is
an i.i.d.\ sum admitting a Rademacher complexity bound.
\item The degenerate terms concentrate around zero following an exponential inequality under appropriate entropy conditions.
\end{enumerate}
To consolidate the strategy given above, we describe what are V- and U-statistics, and how they relate to each other.
These estimators emerge when we allow re-using the same data point repeatedly from a single sample
to estimate a function which takes multiple data points.
\paragraph{V-statistic.}
\label{sec:orga783ed8}
For a given regular statistical functional of degree \(D\) \cite{LeeUStatistics1990}:
\begin{equation}\label{eq:v-estimator-asymptotic-target}\begin{split}
\cQ^D \ke := \int \ke(\e_1, \cdots, \e_D) \cq(\e_1) \cdots \cq(\e_D) \de_1 \cdots \de_D,
\end{split}\end{equation}
its associated von-Mises statistic (V-statistic) is the following quantity \cite{LeeUStatistics1990}:
\begin{equation*}\begin{split}
\Vn{D} \ke := \frac{1}{n^D} \sum_{i_1 =1}^n \cdots \sum_{i_D=1}^n \ke(\S_{i_1}, \ldots, \S_{i_D}).
\end{split}\end{equation*}
Note that Eq.~\eqref{eq:v-estimator-asymptotic-target} does not coincide with the expectation of \(\Vn{D} \ke\) in general, i.e., the V-statistic is generally not an unbiased estimator.
However, it is known to be a consistent estimator of Eq.~\eqref{eq:v-estimator-asymptotic-target} \citep{LeeUStatistics1990}.
\paragraph{U-statistic.}
\label{sec:org02a92bc}
Similarly, for a \(j\)-variate symmetric and integrable function \(h(x_1, \ldots, x_j)\), its corresponding U-statistic \citep{LeeUStatistics1990} of degree \(j\) is
\begin{equation*}\begin{split}
\Un{j} h := \jnCombsetMean h(\e_{\rho(1)}, \ldots, \e_{\rho(j)}). 
\end{split}\end{equation*}
The V- and U-statistics are generalizations of the sample mean (which is the U- and V-statistics of degree \(1\)).
The important difference from the sample mean in higher degrees is that the summands may not be independent.
To deal with the dependence, the following standard decompositions have been developed \cite{LeeUStatistics1990}.
\begin{lemma}[Decomposition of a V-statistic \citep{LeeUStatistics1990}]
A V-statistic can be expressed as a sum of U-statistics of degrees from \(1\) to \(D\) \citep[Section 4.2, Theorem 1]{LeeUStatistics1990}:
\begin{equation*}\begin{split}
\Vn{D}\ke &= \sum_{j=1}^D w_j \Un{j} \Uke{j}
\end{split}\end{equation*}
where the weights \(w_j\) and \(j\)-variate functions \(\Uke{j}\) are
\begin{equation*}\begin{split}
w_j &:= \frac{1}{n^D} \jGroupSplitCard \comb{n}{j},\quad \Uke{j}(\e_1, \ldots, \e_j) := \jGroupSplitMean \ke(\e_{\tau(1)}, \ldots, \e_{\tau(D)}).
\end{split}\end{equation*}
\label{lem:decompose-v-estimator}
\end{lemma}
\begin{proof}
See \citep[Section 4.2, Theorem 1 (p.183)]{LeeUStatistics1990}.
\end{proof}

\begin{remark}[]
The weights \(\{w_j\}_{j=1}^D\) satisfy \(\sum_j w_j = 1\) \citep[Section 4.2, Theorem 1 (p.183)]{LeeUStatistics1990}.
We can also find the order of \(w_j\) with respect to \(n\) as:
\begin{equation*}\begin{split}
w_D = \frac{1}{n^D} \annot{\jGroupSplitCard}{\(D!\)} \comb{n}{D} = \frac{n (n-1) \cdots (n-D + 1)}{n^D} = \Order(1), \quad w_j = \Order(n^{-(D-j)}), \quad \Uke{D} = \ke.
\end{split}\end{equation*}
\end{remark}
\begin{lemma}[Hoeffding decomposition of a U-statistic {\citep[p.449]{ShermanMaximal1994}}]
A U-statistic with a symmetric kernel \(\psi\) can be decomposed as a sum of U-statistics of degrees from \(1\) to \(D\) as
\begin{equation}\label{eq:theory:hoeffding-decomposition-general}\begin{split}
\Un{D} \psi - \cEData \Un{D}\psi &= \sum_{j=1}^D \Un{j} \psi_j \\
&= \hEData \psi_1 + \sum_{j=2}^D \Un{j} \psi_j
\end{split}\end{equation}
where \(\{\psi_j\}_{j=1}^D\) are \(j\)-variate, symmetric and degenerate functions.
Note that \(\cEData \Un{D}\psi = \ED \psi\).
Here, a \(j\)-variate symmetric function \(\psi_j\) is said \emph{degenerate} when
\begin{equation*}\begin{split}
\forall s_2, \ldots, s_j, \quad \psi_j(\cQ, s_2, \ldots, s_j) = 0.
\end{split}\end{equation*}
Specifically, \(\psi_1\) is
\begin{equation}\label{eq:theory:hoeffding-first-kernel}\begin{split}
\psi_1(s) &= \psi(s, \cQ, \ldots, \cQ) + \cdots + \psi(\cQ, \ldots, \cQ, s) - D \cQ^D \psi \\
&= D \cdot (\psi(s, \cQ, \ldots, \cQ) - \cQ^D \psi) \quad(\text{by symmetry}).
\end{split}\end{equation}
\label{lem:hoeffding-decomposition}
\end{lemma}
For further details, see \citep[p.449]{ShermanMaximal1994}.
Note that in \citep[p.449]{ShermanMaximal1994}, Eq.~\eqref{eq:theory:hoeffding-decomposition-general} is written using \(\ED \psi\) in place of \(\cEData \Un{D}\psi\).
This is because
\begin{equation*}\begin{split}
\cEData \Un{D}\psi = \Un{D} \cEData \psi = \Un{D} \ED \psi = \ED \psi
\end{split}\end{equation*}
holds by linearity and symmetry.
\begin{remark}[Connecting the lemmas to Section~\ref{sec:theory:pseudo-generalization-bound}]
It can be easily checked by definition that the proposed risk estimator Eq.~\eqref{eq:theory:proposed-risk-estimator-1} takes the form of a V-statistic: \(\cbarhR(\g) = \Vn{D} \ke\) for each \(\g \in \G\).
Let us denote \(\leadingKe(\s) := \ke(\s, \cQ, \ldots, \cQ)\).
Then \(\cEData \leadingKe = \ED \ke\) holds by definition.
Substituting these into Eq.~\eqref{eq:theory:hoeffding-first-kernel},
we have that Eq.~\eqref{eq:theory:hoeffding-decomposition-general} applied to \(\psi = \ke\) is equivalent to
\begin{equation*}\begin{split}
\Un{D} \ke - \cEData \Un{D} \ke = D \cdot (\hEData \leadingKe - \cEData \leadingKe) + \sum_{j=2}^D \Un{j} \ke_j.
\end{split}\end{equation*}
where \(\{\ke_j\}_{j=2}^D\) are symmetric degenerate functions.
In Section~\ref{sec:theory:pseudo-generalization-bound}, we first decompose \(\cbarhR(\g)\) into a sum of U-statistics.
After such conversion, we take a closer look at the leading term, \(\hEData \leadingKe\).
\label{theory:rem:connect-hoeffding-decomp-and-bound}
\end{remark}
\subsection{Proof of pseudo estimation error bound \label{sec:theory:pseudo-generalization-bound}}
\label{sec:orge42bb4b}
\begin{proof}[(Proof of Lemma~\ref{lem:bound-pseudo-generalization-error})]
First, we have
\begin{equation*}\begin{split}
\cbarR(\cbarhg) - \cbarR(\cbarg) = \cbarR(\cbarhg) - \cbarhR(\cbarhg) + \cbarhR(\cbarhg) - \cbarR(\cbarg) &\leq \cbarR(\cbarhg) - \cbarhR(\cbarhg) + \cbarhR(\cbarg) - \cbarR(\cbarg) \\
&\leq 2 \supg \left|\cbarhR(\g) - \cbarR(\g)\right|.
\end{split}\end{equation*}
Now the right-most expression can be decomposed as
\begin{equation*}\begin{split}
&\supg |\cbarhR(\g) - \cbarR(\g)| = \supg |\Vn{D} \ke - \cEData \Vn{D} \ke| \\
&\leq w_D\supg \left|\Un{D}\ke - \cEData\Un{D}\ke\right| + \sum_{j=1}^{D-1} w_j \supg \left|\Un{j}\Uke{j} - \cEData \Un{j} \Uke{j}\right| \quad(\because \text{Lemma~\ref{lem:decompose-v-estimator}})\\
&\leq w_D\supg |\Un{D}\ke - \cEData\Un{D}\ke| + 2 \LossUpperBound \sum_{j=1}^{D-1} w_j \\
&\leq w_D\left(\supg |\hEData \dUke{1}| + \sum_{j=2}^D \supg |\Un{j}\dUke{j}|\right) + 2 \LossUpperBound \sum_{j=1}^{D-1} w_j \quad(\because \text{Lemma~\ref{lem:hoeffding-decomposition}})\\
&= w_D\left(\annot{\supg |\hEData D (\leadingKe - \cEData \leadingKe)|}{Addressed in Lemma~\ref{lem:rademacher-U-leading-term}} + \annot{\sum_{j=2}^D \supg |\Un{j}\dUke{j}|}{Addressed in Lemma~\ref{lem:rademacher-U-degenerate-part}}\right) + 2 \LossUpperBound \sum_{j=1}^{D-1} w_j.
\end{split}\end{equation*}
where \(\dUke{j}\) are symmetric degenerate functions and \(\leadingKe\) is defined as in Remark~\ref{theory:rem:connect-hoeffding-decomp-and-bound}.
Applying Lemma~\ref{lem:rademacher-U-leading-term} to the first term and Lemma~\ref{lem:rademacher-U-degenerate-part} to the second term with the union bound, we obtain the assertion.
\end{proof}
In the last part of the proof we used the following lemmas.
Because the leading term is an i.i.d.\ sum, the following Rademacher complexity bound can be proved.
\begin{lemma}[U-process bound: the leading term]
Assume Assumption~\ref{assum:q-bdd-and-lipschitz} holds.
Then, we have with probability at least \(1 - \delta\),
\begin{equation*}\begin{split}
\supg |\hEData (\leadingKe - \cEData \leadingKe)| \leq 2 \Rad(\G) + \LossUpperBound \sqrt{\frac{\log(2/\delta)}{2n}},
\end{split}\end{equation*}
where \(\Rad\) is defined in Definition~\ref{def:rademacher-complexity}.
\label{lem:rademacher-U-leading-term}
\end{lemma}
\begin{proof}
Applying the standard one-sided Rademacher complexity bound based on McDiarmid's inequality \citep[Theorem~3.1]{MohriFoundations2012} twice with the union bound, we obtain the lemma.
\end{proof}
The other terms than the leading term are degenerate U-statistics, hence the following holds under appropriate entropy assumptions.
\begin{lemma}[U-process bound: degenerate terms {\citep[Corollary~7]{ShermanMaximal1994}}]
Assume Assumption~\ref{assum:euclidean-class}. Then for each \(j = 2, \ldots, D\),
there exist constants \(\PseudoGenCj\) such that for any \(\delta \in (0, 1)\), we have with probability at least \(1 - \delta' / (D - 1)\),
\begin{equation*}\begin{split}
\supg |\Un{j}\dUke{j}| \leq \frac{(D-1)}{\delta'}\PseudoGenCj n^{-j/2}
\end{split}\end{equation*}
where \(\PseudoGenCj\) depends only on \(A\), \(V\), and \(\LossUpperBound\).
\label{lem:rademacher-U-degenerate-part}
\end{lemma}
\begin{proof}
The proof follows a similar path as that of \citep[Corollary~7]{ShermanMaximal1994}, but we provide more explicit expressions to inspect the order with respect to \(n\).
Let \(\UkeSp{j} := \{\dUke{j} : \g \in \G\}\).
Then \(\UkeSp{j}\) is Euclidean for an envelope \(F_j\) satisfying \(\cQ^j F_j^2 < \infty\)
by Lemma~6 in \citet{ShermanMaximal1994} and Assumption~\ref{assum:euclidean-class}.
In addition, \(\UkeSp{j}\) is a set of functions degenerate with respect to \(\cQ\).
Without loss of generality, we can take \(F_j\) such that \(F_j \leq \LossUpperBound\).
Similarly to the proof of \citep[Corollary~4]{ShermanMaximal1994}, we can apply \citep[Main~Corollary]{ShermanMaximal1994} with \(p = 1\) in their notation to obtain
\begin{equation*}\begin{split}
\cEData \supg |n^{j/2} \Un{j}\dUke{j}| \leq \Gamma A^{1/2mp}(\cQ^j F_j^2)^{(\epsilon + \alpha)/2} \leq \annot{\Gamma A^{1/2mp}(\LossUpperBound)^{\epsilon + \alpha}}{\(=: \PseudoGenCj\)}
\end{split}\end{equation*}
where \(\Gamma\) is a universal constant \citep[Main~Corollary]{ShermanMaximal1994},
\(\epsilon \in (0, 1)\) and \(m\) are chosen to satisfy \(1 - V/2m > 1 - \epsilon\),
and \(\alpha = 1 - V/2m\).
By applying Markov inequality, we have for arbitrary \(u > 0\),
\begin{equation*}\begin{split}
\Probability{\cData}\left(\supg |n^{j/2} \Un{j}\dUke{j}| > u\right) \leq \frac{\PseudoGenCj}{u},
\end{split}\end{equation*}
where \(\Probability{\cData}(E)\) denotes the probability of the event \(E\) with respect to \(\cData\).
Equating the right hand side with \(\delta' / (D - 1)\) and solving for \(u\), we obtain the result.
\end{proof}
\subsection{Proof of approximation error bound \label{sec:theory:approx-error-bound}}
\label{sec:org7619785}
\begin{proof}[(Proof of Lemma~\ref{lem:approximation-error})]
Due to Lemma~\ref{lem:decompose-v-estimator}, we have
\begin{equation*}\begin{split}
\supg \left(\R(\g) - \cbarR(\g)\right) &= \supg \left(\R(\g) - \cEData\Vn{D}\ke\right) \\
&= \supg \left(\sum_{j=1}^D w_j (\R(\g) - \cEData\Un{j}\Uke{j})\right)\\
&\leq w_D \supg \left(\R(\g) - \cEData \Un{D}\Uke{D}\right) + 2 \LossUpperBound \sum_{j=1}^{D-1}\annot{w_j}{\(\Order(n^{-(D-j)})\)} \\
&\leq w_D \supg \left(\R(\g) - \cEData \Un{D}\Uke{D}\right) + 2 \LossUpperBound \Order(n^{-1})
\end{split}\end{equation*}
By applying Lemmas~\ref{lem:approx-error-bound} (with \(j = D\)), we obtain
\begin{equation*}\begin{split}
&\supg \left(\R(\g) - \cEData\Un{D}\Uke{D}\right) \leq \supg \LoneWith{\q}{\l(\f(\g, \cdot)) - \l(\g, \hf(\cdot))} + D \LossUpperBound \Lone{\q - \cq}.\\
\end{split}\end{equation*}
The right-hand side can be further bounded by applying Lemmas~\ref{lem:approx-density-difference} and \ref{lem:approx-loss-difference} by
\begin{equation*}\begin{split}
&\ApproxLossDifferenceBound + D \LossUpperBound \left(\ApproxDensityDifferenceBoundMain + \qUpperBound \ApproxDensityDifferenceBoundRemainder\right) \\
&\leq \ApproxErrorUpperBoundContent
\end{split}\end{equation*}
and hence the assertion of the lemma.
\end{proof}
The above proof combined three approximation bounds, which are shown in the following lemmas.
The following lemma reduces the difference in the expectation of U-statistic into the differences in the loss function and the density function.
Although we apply the following Lemma~\ref{lem:approx-error-bound} only with \(j=D\), we prove its general form for \(j \in [D]\).
\begin{lemma}[Approximation bound for U-statistic of degree-\(j\)]
Fix \(j \in [D]\). Assume Assumption~\ref{assum:q-bdd-and-lipschitz}.
Then we have for any \(\g \in \G\),
\begin{equation*}\begin{split}
\R(\g) - \cEData\Un{j}\Uke{j} \leq \LoneWith{\q}{\l(\g, \f(\cdot)) - \l(\g, \hf(\cdot))} + j \LossUpperBound \Lone{\q - \cq}
\end{split}\end{equation*}
\label{lem:approx-error-bound}
\end{lemma}
\begin{proof}
Let us define a \(D\)-variate function \(\lstar\) and a \(j\)-variate function \(\lstarj\) (similarly to \(\ke\) and \(\Ukej\), respectively) by
\begin{equation*}\begin{split}
\lstar(\s_1, \ldots, \s_D) &:= \DSymmetricMean \l(\g, \f(\s_{\pi(1)}^{(1)}, \ldots, \s_{\pi(D)}^{(D)})),\\
\lstarj(\s_1, \ldots, \s_j) &:= \jGroupSplitMean \lstar(\s_{\tau(1)}, \ldots, \s_{\tau(D)}).
\end{split}\end{equation*}
Then, recalling \(\Q \in \Qsp\), we can show \(\R(\g) = \Qn(\Unj\lstarj)\) because
\begin{equation*}\begin{split}
\Qn(\Unj \lstarj) &= \Qn(\jnCombsetMean \lstarj(\S_{\rho(1)}, \ldots, \S_{\rho(j)})) \\
&= \Qn(\jnCombsetMean \jGroupSplitMean \lstar(\S_{\rho \circ \tau(1)}, \ldots, \S_{\rho \circ \tau(D)})) \\
&= \Qn(\jnCombsetMean \jGroupSplitMean \DSymmetricMean \l(\g, \f(\S_{\rho \circ \tau \circ \pi(1)}^{(1)}, \ldots, \S_{\rho \circ \tau \circ \pi(D)}^{(D)}))) \\
&= \jnCombsetMean \jGroupSplitMean \DSymmetricMean \Qn \l(\g, \f(\S_{\rho \circ \tau \circ \pi(1)}^{(1)}, \ldots, \S_{\rho \circ \tau \circ \pi(D)}^{(D)})) \\
&= \jnCombsetMean \jGroupSplitMean \DSymmetricMean \Q [\l(\g, \f(\S^{(1)}, \ldots, \S^{(D)}))] \quad (\because \Q \in \Qsp)\\
&= \jnCombsetMean \jGroupSplitMean \DSymmetricMean \R(\g) = \R(\g).
\end{split}\end{equation*}
Combining this expression with Lemma~\ref{lem:decompose-v-estimator},
\begin{equation*}\begin{split}
\R(\g) - \cEData\Unj\Ukej &= \Qn(\Unj \lstarj) - \cQn (\Unj\Ukej) \\ 
&= \annot{\Qn(\Unj \lstarj - \Unj\Ukej)}{\(A\)} + \annot{(\Qn - \cQn)(\Unj\Ukej)}{\(B\)} \\
\end{split}\end{equation*}
Now, \(A\) can be bounded from above as
\begin{equation*}\begin{split}
A &= \Qn(\Unj \lstarj - \Unj\Ukej) \\
&= \jnCombsetMean \jGroupSplitMean \DSymmetricMean \Qn (\l(\g, \f(\S^{(1)}_{\rho\circ\tau\circ\pi(1)}, \ldots, \S^{(D)}_{\rho\circ\tau\circ\pi(D)})) - \l(\g, \hf(\S^{(1)}_{\rho\circ\tau\circ\pi(1)}, \ldots, \S^{(D)}_{\rho\circ\tau\circ\pi(D)}))) \\
&= \jnCombsetMean \jGroupSplitMean \DSymmetricMean \Q (\l(\g, \f(\S^{(1)}, \ldots, \S^{(D)})) - \l(\g, \hf(\S^{(1)}, \ldots, \S^{(D)}))) \quad (\because \Q \in \Qsp)\\
&\leq \LoneWith{\q}{\l(\g, \f(\cdot)) - \l(\g, \hf(\cdot))}
\end{split}\end{equation*}
Then recalling Assumption~\ref{assum:q-bdd-and-lipschitz}, we can bound \(B\) from above as
\begin{equation*}\begin{split}
B &= (\Qn - \cQn)(\Unj\Ukej) = (\Qn - \cQn)\left(\jnCombsetMean \Ukej(\S_{\rho(1)}, \ldots, \S_{\rho(j)})\right) \\
&= \jnCombsetMean (\Qn - \cQn)\left(\Ukej(\S_{\rho(1)}, \ldots, \S_{\rho(j)})\right) = (\Qj - \cQj) (\Ukej(\S_1, \ldots, \S_j)) \quad (\because \text{ symmetry})\\
&\leq \LossUpperBound \int \left|\prod_{i=1}^j \q(\s_i) - \prod_{i=1}^j \cq(\s_i)\right| \ds_1 \cdots \ds_j \\
&= \LossUpperBound \int \left|\sum_{i=1}^j \q(\s_1)\cdots\q(\s_{i-1}) \cdot (\q(\s_i) - \cq(\s_i)) \cdot \cq(\s_{i+1})\cdots\cq(\s_j)\right| \ds_1 \cdots \ds_j \\
&\leq \LossUpperBound \sum_{i=1}^j \int \q(\s_1)\cdots\q(\s_{i-1}) \cdot |\q(\s_i) - \cq(\s_i)| \cdot \cq(\s_{i+1})\cdots\cq(\s_j) \ds_1 \cdots \ds_j \\
&= \LossUpperBound \sum_{i=1}^j \int |\q(\s_i) - \cq(\s_i)| \ds_i = \LossUpperBound \cdot j \Lone{\q - \cq}, \\
\end{split}\end{equation*}
which proves the assertion.
\end{proof}
Now the following lemmas bound each approximation terms in terms of the difference between \(\f\) and \(\hf\).
\begin{lemma}[Loss difference approximation]
Assume Assumption~\ref{assum:loss-uniform-lipschitz}.
Then we have for any \(\g \in \G\),
\begin{equation*}\begin{split}
\LoneWith{\q}{\l(\g, \f(\cdot)) - \l(\g, \hf(\cdot))} \leq \ApproxLossDifferenceBound
\end{split}\end{equation*}
\label{lem:approx-loss-difference}
\end{lemma}
\begin{proof}
\begin{equation*}\begin{split}
&\LoneWith{\q}{\l(\g, \f(\cdot)) - \l(\g, \hf(\cdot))} = \int |\l(\g, \f(\e)) - \l(\g, \hf(\e))|\q(\e) \de \\
&\qquad\leq \qUpperBound \int \lossUniformLipschitzConst \ltwo{\f(\e) - \hf(\e)} \de \\
&\qquad\leq \qUpperBound \lossUniformLipschitzConst\int \lone{\f(\e) - \hf(\e)} \de \leq \qUpperBound \lossUniformLipschitzConst \sum_{j=1}^D \SobolevOne{\f_j - \hf_j}. \\
\end{split}\end{equation*}
\end{proof}
\begin{lemma}[Density difference approximation]
Assume Assumptions~\ref{assum:q-bdd-and-lipschitz}, \ref{assum:f-lipschitz-and-holder}, and \ref{assum:dfinv-operator-bounded}.
Then we have
\begin{equation*}\begin{split}
\Lone{\q - \cq} \leq \ApproxDensityDifferenceBoundMain + \qUpperBound \ApproxDensityDifferenceBoundRemainder\\
\end{split}\end{equation*}
where \(\ApproxDensityDifferenceJacobianPieceCoeffOne\) and \(\ApproxDensityDifferenceBoundRemainder\) are defined as in Lemma~\ref{lem:approx-density-bound-2}.
\label{lem:approx-density-difference}
\end{lemma}
\begin{proof}
Since \(\cq(\e) = \q(\finvhf(\e))\J{\finvhf}{\e}\), we have
\begin{equation*}\begin{split}
\Lone{\q - \cq} &= \int \left|\q(\e) - \q(\finvhf(\e)) \J{\finvhf}{\e}\right|\de \\
&\leq \int |\q(\e) - \q(\finvhf(\e))| \de + \int \q(\finvhf(\e)) \left|1 - \J{\finvhf}{\e}\right| \de \\
&\leq \annot{\int |\q(\e) - \q(\finvhf(\e))| \de}{(A)} + \qUpperBound \annot{\int \left|1 - \Ji{\finvhf}{\e}\right| \de}{(B)} \\
\end{split}\end{equation*}
where the last line follows from the triangle inequality.
Applying Lemma~\ref{lem:approx-density-bound-1} to (A) and Lemma~\ref{lem:approx-density-bound-2} to (B) yields the assertion.
\end{proof}
\begin{lemma}[]
Assume Assumptions~\ref{assum:q-bdd-and-lipschitz} and \ref{assum:f-lipschitz-and-holder}.
Then,
\begin{equation*}\begin{split}
\int |\q(\e) - \q(\finvhf(\e))| \de \leq \ApproxDensityBoundOne
\end{split}\end{equation*}
\label{lem:approx-density-bound-1}
\end{lemma}
\begin{proof}
We have
\begin{equation*}\begin{split}
&\int |\q(\e) - \q(\finvhf(\e))| \de = \int |\q(\finvf(\e)) - \q(\finvhf(\e))| \de \\
&\qquad\leq \qLipschitzConst \finvLipschitzConst \int \ltwo{\f(\e) - \hf(\e)} \de \leq \qLipschitzConst \finvLipschitzConst \int \lone{\f(\e) - \hf(\e)} \de \\
&\qquad\leq \qLipschitzConst \finvLipschitzConst \sum_{j=1}^D \SobolevOne{\f_j - \hf_j} \\
\end{split}\end{equation*}
\end{proof}
\begin{lemma}[Jacobian difference approximation]
Assume Assumptions~\ref{assum:q-bdd-and-lipschitz} and \ref{assum:dfinv-operator-bounded}.
Then,
\begin{equation*}\begin{split}
\int \left|1 - \Ji{\finvhf}{\e}\right| \de &\leq \ApproxDensityBoundTwo + \ApproxDensityDifferenceBoundRemainder,
\end{split}\end{equation*}
where
\begin{equation*}\begin{split}
\ApproxDensityDifferenceJacobianPieceCoeff &:= \ApproxDensityDifferenceJacobianPieceCoeffDef, \\
\ApproxDensityDifferenceBoundRemainder &:= \ApproxDensityDifferenceBoundRemainderDef.\\
\end{split}\end{equation*}
\label{lem:approx-density-bound-2}
\end{lemma}
\begin{proof}
Applying Lemma~\ref{lem:det-difference-bound} with \(A := \Ji{\finvf}{\e} = I\), we obtain
\begin{equation*}\begin{split}
\int \left|1 - \Ji{\finvhf}{\e}\right| \de &= \int \left|\Ji{\finvf}{\e} - \Ji{\finvhf}{\e}\right| \de \\ 
&\leq \int \sum_{d=1}^D \comb{D}{d} \op{\dfinvf{\e}(\e) - \dfinvhf{\e}(\e)}^d \de.\\
\end{split}\end{equation*}
Now, each term in the integrand can be bounded from above as
\begin{equation*}\begin{split}
&\op{\dfinvf{\e}(\e) - \dfinvhf{\e}(\e)} \\
&= \op{\left(\dfinv{z}(\f(\e))\right) \left(\df{\e}(\e)\right) - \left(\dfinv{z}(\hf(\e))\right) \left(\dhf{\e}(\e)\right)} \\
&\leq \op{\left(\dfinv{z}(\f(\e)) - \dfinv{z}(\hf(\e))\right)\left(\df{\e}(\e)\right)} + \op{\left(\dfinv{z}(\hf(\e))\right) \left(\df{\e}(\e) - \dhf{\e}(\e)\right)} \\
&\leq \op{\dfinv{z}(\f(\e)) - \dfinv{z}(\hf(\e))}\op{\df{\e}(\e)} + \op{\dfinv{z}(\hf(\e))}\op{\df{\e}(\e) - \dhf{\e}(\e)} \\
&\qquad\quad(\because \text{ submultiplicativity \citep[Section~2.3.2]{GolubMatrix2013}})\\
&\leq \op{\dfinv{z}(\f(\e)) - \dfinv{z}(\hf(\e))}\left(D\cdot\supnrm{\df{\e}(\e)}\right) + \left(D\cdot \supnrm{\dfinv{z}(\hf(\e))}\right) \op{\df{\e}(\e) - \dhf{\e}(\e)}\\
&\qquad\quad (\because \op{\cdot} \leq D \msupnrm{\cdot} \text{\citep[Section~2.3.2]{GolubMatrix2013}})\\
&\leq \op{\dfinv{z}(\f(\e)) - \dfinv{z}(\hf(\e))}\cdot\left(D \dfSupNormConst\right) \ +\  \left(D\dfinvSupNormConst\right) \cdot \op{\df{\e}(\e) - \dhf{\e}(\e)}\\
&\leq D \dfSupNormConst \sqrt{D} \opone{\dfinv{z}(\f(\e)) - \dfinv{z}(\hf(\e))} + D\dfinvSupNormConst \sqrt{D} \opone{\df{\e} - \dhf{\e}} \\
&\qquad\quad (\because \op{\cdot} \leq \sqrt{D} \opone{\cdot} \text{\citep[Section~2.3.1]{GolubMatrix2013}})\\
&= D^{\frac{3}{2}} \dfSupNormConst \maxk \sumj \left|\pfinvj{z_k}(\f(\e)) - \pfinvj{z_k}(\hf(\e))\right| + D^{\frac{3}{2}} \dfinvSupNormConst \maxk \sumj \left|\pdfj{\e_k}(\e) - \pdhfj{\e_k}(\e)\right| \\
&\leq D^{\frac{3}{2}} \dfSupNormConst \maxk \sumj \finvHolderNorm{\finv_j}\ltwo{\f(\e) - \hf(\e)} + D^{\frac{3}{2}}\dfinvSupNormConst \sumk \sumj \left|\pdfj{\e_k}(\e) - \pdhfj{\e_k}(\e)\right| \\
&\leq D^{\frac{3}{2}} \dfSupNormConst \left(\sumj \finvHolderNorm{\finv_j}\right)\lone{\f(\e) - \hf(\e)} + D^{\frac{3}{2}} \dfinvSupNormConst \sumk \sumj \left|\pdfj{\e_k}(\e) - \pdhfj{\e_k}(\e)\right| \\
&\qquad\quad (\because \ltwo{\cdot} \leq \lone{\cdot} \text{\citep[Section~2.2.2]{GolubMatrix2013}}).\\
\end{split}\end{equation*}
When powered to \(d\), this yields
\begin{equation*}\begin{split}
&\op{\dfinvf{\e}(\e) - \dfinvhf{\e}(\e)}^d \\
&\leq (D^2+D)^{d-1} \biggl[\sumj \left(D^{3/2}\dfSupNormConst \left(\sumk \finvHolderNorm{\finv_k}\right)|\f_j(\e) - \hf_j(\e)|\right)^d \\
&\qquad\qquad\qquad\qquad+ \sumk \sumj \left(D^{3/2} \dfinvSupNormConst \left|\pdfj{\e_k}(\e) - \pdhfj{\e_k}(\e)\right|\right)^d \biggr]
\end{split}\end{equation*}
where we used \((\sum_{i=1}^{L} a_i)^d \leq L^{d-1}(\sum_{i=1}^{L} a_i^d)\) for \(a_i \geq 0\), which follows from \Holder inequality.
Hence,
\begin{equation*}\begin{split}
&\int \op{\dfinvf{\e}(\e) - \dfinvhf{\e}(\e)}^d \de \\
&\leq D^{\frac{5}{2}d - 1}(D+1)^{d-1} \biggl[\left(\dfSupNormConst \sumk \finvHolderNorm{\finv_k}\right)^d \sumj \int |\f_j(\e) - \hf_j(\e)|^d \de \\
&\qquad\qquad\qquad\qquad\qquad\qquad+ (\dfinvSupNormConst)^d \sumj \left(\sumk \int \left|\pdfj{\e_k}(\e) - \pdhfj{\e_k}(\e)\right|^d \de \right) \biggr]\\
&\leq (D+1)^{\frac{7}{2}d-2} \left((\dfSupNormConst)^d \left(\sumk \finvHolderNorm{\finv_k}\right)^d \sumj \Sobolevd{\f_j - \hf_j}^d + (\dfinvSupNormConst)^d \sumj \Sobolevd{\f_j - \hf_j}^d \right)\\
&\leq \ApproxDensityDifferenceJacobianPieceCoeff \sum_{j=1}^D \Sobolevd{\f_j - \hf_j}^d.
\end{split}\end{equation*}
Therefore,
\begin{equation*}\begin{split}
&\int \left|1 - \Ji{\finvhf}{\e}\right| \de \\
&\leq \sum_{d=1}^D \comb{D}{d} \int \op{\dfinvf{\e}(\e) - \dfinvhf{\e}(\e)}^d \de\\
&\leq \ApproxDensityBoundTwo + \annot{\ApproxDensityDifferenceBoundRemainderDef}{\(\ApproxDensityDifferenceBoundRemainder\)}
\end{split}\end{equation*}
\end{proof}
Lemma~\ref{lem:approx-density-bound-2} used the following lemma to bound the difference in Jacobian determinants.
\begin{lemma}[Determinant perturbation bound {\citep[Corollary 2.11]{IpsenPerturbation2008}}]
Let \(A\) and \(E\) be \(D \times D\) complex matrices. Then,
\begin{equation*}\begin{split}
|\det (A) - \det(A+E)| \leq \sum_{d=1}^D \comb{D}{d} \op{A}^{D - d}\op{E}^d.
\end{split}\end{equation*}
\label{lem:det-difference-bound}
\end{lemma}
\subsection{Comparison of Rademacher complexities \label{sec:theory:comparison-of-rademacher}}
\label{sec:orgbaa0fea}
The following consideration demonstrates how the effective complexity measure \(\Rad\) in Theorem~\ref{thm:generalization-error}
resulting from the proposed method may enjoy a relaxed dependence on the input dimensionality compared to the ordinary empirical risk minimization.
To do so, we first recall the definition of the ordinary Rademacher complexity and a standard performance guarantee derived based on it.
\begin{definition}[Ordinary Rademacher complexity]
The ordinary empirical risk minimization finds the candidate hypothesis by
\begin{equation*}\begin{split}
\hg \in \argmin_{\g \in \G} \hR(\g),
\end{split}\end{equation*}
where
\begin{equation*}\begin{split}
\hR(\g) := \meanX{i} \l(\g, \Z_i) = \frac{1}{n}\sum_{i=1}^n \l(\g, \hf(\S_i^{(1)}, \ldots, \S_i^{(D)})) 
\end{split}\end{equation*}
and the corresponding ordinary Rademacher complexity \(\ORad(\G)\) is
\begin{equation*}\begin{split}
\ORad(\G) &:= \frac{1}{n} \cEData\Erad\left[\supg \left|\sum_{i=1}^{n} \rad_i \l(\S_i^{(1)}, \ldots, \S_i^{(D)})\right|\right]
\end{split}\end{equation*}
where \(\{\rad_i\}_{i=1}^n\) are independent uniform sign variables and we denoted \(\l(\s^{(1)}, \ldots, \s^{(D)})= \l(\g, \hf(\s^{(1)}, \ldots, \s_i^{(D)}))\) by abuse of notation.
This yields the standard Rademacher complexity based bound. Applying Lemma~\ref{lem:rademacher-U-leading-term} and using the same proof technique, we have that with probability at least \(1 - \delta\),
\begin{equation*}\begin{split}
\R(\hg) - \R(\gstar) \leq 2 \supg |\R(\g) - \hR(\g)| \leq 4 \ORad(\G) + 2\LossUpperBound \sqrt{\frac{\log(2/\delta)}{2n}}.
\end{split}\end{equation*}
Therefore, we the corresponding complexity terms are \(\ORad(\G)\) and \(D \Rad(\G)\). In Remark~\ref{theory:remark:comparison-of-rademacher}, we make a comparison of these two complexity measures by taking an example.
To recall, the effective Rademacher complexity can be written as, in terms of the notation in this section,
\begin{equation*}\begin{split}
&\Rad(\G) = \frac{1}{n} \cEData\Erad\left[\supg \left|\sum_{i=1}^{n} \rad_i \ETwoToD{\S'} \ke(\S_i, \S_2', \ldots, \S_D')\right|\right]\\
&= \frac{1}{n} \cEData\Erad\left[\supg \left|\sum_{i=1}^{n} \rad_i \EOneToD{\S'} \frac{1}{D} \left(\l(\S_i^{(1)}, \S_2'^{(2)}, \ldots, \S_D'^{(D)}) + \cdots + \l(\S_1'^{(1)}, \S_2'^{(2)}, \ldots, \S_i^{(D)})\right)\right|\right]\\
\end{split}\end{equation*}
\label{def:ordinary-rademacher-complexity}
\end{definition}
\begin{remark}[Comparison of Radmacher complexities]
As an example, consider \(\CubeLipschitz\), the set of \(L\)-Lipschitz functions (with respect to infinity norm) on the unit cube \([0, 1]^d\).
It is well-known that there exists a constant \(C > 0\) such that the following holds \citep[Example~5.10, p.129]{WainwrightHighDimensional2019} for sufficiently small \(t > 0\):
\begin{equation}\label{eq:theory:remark:lipschitz-metric-entropy}\begin{split}
\log \metricEntropy{t}{\CubeLipschitz}{\supnrm{\cdot}} \asymp (C / t)^d.
\end{split}\end{equation}
Here, \(a(t) \asymp b(t)\) indicates that there exist \(k_1, k_2 > 0\) such that, for sufficiently small \(t\), it holds that \(k_1 b(t) \leq a(t) \leq k_2 b(t)\).
On the other hand, the well-known discretization argument implies that there exist constants \(c\) and \(B\) such that for any \(t \in (0, B]\), the following relation between the Rademacher complexity and the metric entropy holds:
\begin{equation}\label{eq:theory:remark:rademacher-entropy-relation}\begin{split}
\ORad(\CubeLipschitz) \leq t + c \sqrt{\frac{\log \metricEntropy{t}{\CubeLipschitz}{\supnrm{\cdot}}}{n}}.
\end{split}\end{equation}
Substituting Eq.~\eqref{eq:theory:remark:lipschitz-metric-entropy} into Eq.~\eqref{eq:theory:remark:rademacher-entropy-relation},
we can find that, for large enough \(n\), the right hand side is minimized at \(t = (c \cdot C^{\frac{d}{2}} \cdot \frac{d}{2})^{\frac{2}{2 + d}} \cdot n^{- \frac{1}{2 + d}}\).
This yields
\begin{equation}\label{eq:theory:rademacher-dimension-dependence}\begin{split}
\ORad(\CubeLipschitz) \leq \tilde C \cdot n^{-\frac{1}{2+d}}
\end{split}\end{equation}
with a new constant \(\tilde C = \left(c \cdot C^{\frac{d}{2}} \cdot \frac{d}{2}\right)^{\frac{2}{2 + d}} + c \cdot C^{\frac{d}{2}}\left(c \cdot C^{\frac{d}{2}} \cdot \frac{d}{2}\right)^{-\frac{d}{2 + d}}\).
Therefore, by substituting \(d = D\) in Eq.~\eqref{eq:theory:rademacher-dimension-dependence}, the metric-entropy based bound on the ordinary Rademacher complexity exhibits exponential dependence on the input dimension as
\begin{equation*}\begin{split}
\ORad(\CubeLipschitz) &\leq \Order\left(n^{-\frac{1}{2 + D}}\right),
\end{split}\end{equation*}
which is a manifestation of the curse of dimensionality.
On the other hand, by following a similar calculation, we can see that the effective Rademacher complexity \(\Rad(\CubeLipschitz)\) avoids an exponential dependence on the input dimension \(D\).
By substituting \(d = 1\) in Eq.~\eqref{eq:theory:rademacher-dimension-dependence}, we can see
\begin{equation*}\begin{split}
D \Rad(\CubeLipschitz) &\leq \ORad(\CubeLipschitz_1) + \cdots + \ORad(\CubeLipschitz_D) \leq \Order\left(n^{-\frac{1}{3}}\right), \\
\end{split}\end{equation*}
where \(\CubeLipschitz_j := \{\EOneToD{\S'}h({\S_1'}^{(1)}, \ldots, {\S_{j-1}'}^{(j-1)}, (\cdot)^{(j)}, {\S_{j+1}'}^{(j+1)}, \ldots, {\S'_D}^{(D)}) : h \in \CubeLipschitz\}\).
This is because the Lipschitz constant of functions in \(\CubeLipschitz_j\) is at most \(L\) (i.e., the Lipschitz constant does not increase by the marginalization procedure) because for any \(h \in \CubeLipschitz_j\),
\begin{equation*}\begin{split}
&|h(x) - h(y)| \\
&= |\EOneToD{\S'}[h(\S_1'^{(1)}, \ldots, \S_{j-1}'^{(j-1)}, x, \S_{j+1}'^{(j+1)}, \ldots, \S_D'^{(D)}) - h(\S_1'^{(1)}, \ldots, \S_{j-1}'^{(j-1)}, y, \S_{j+1}'^{(j+1)}, \ldots, \S_D'^{(D)})]| \\
&\leq \EOneToD{\S'} |h(\S_1'^{(1)}, \ldots, \S_{j-1}'^{(j-1)}, x, \S_{j+1}'^{(j+1)}, \ldots, \S_D'^{(D)}) - h(\S_1'^{(1)}, \ldots, \S_{j-1}'^{(j-1)}, y, \S_{j+1}'^{(j+1)}, \ldots, \S_D'^{(D)})| \\
&\leq \EOneToD{\S'} L \cdot \|(\S_1'^{(1)}, \ldots, \S_{j-1}'^{(j-1)}, x, \S_{j+1}'^{(j+1)}, \ldots, \S_D'^{(D)}) - (\S_1'^{(1)}, \ldots, \S_{j-1}'^{(j-1)}, y, \S_{j+1}'^{(j+1)}, \ldots, \S_D'^{(D)})\| \\
&= \EOneToD{\S'} L \cdot \|(0, \ldots, 0, x - y, 0, \ldots, 0)\| \\
&= L \cdot |x - y|. \\
\end{split}\end{equation*}
\label{theory:remark:comparison-of-rademacher}
\end{remark}
\subsection{Remark on higher order Sobolev norms}
\label{sec:org8f5d375}
Here, we comment on how the term \(\ApproxDensityDifferenceBoundRemainder\) is treated as a higher order term of \(\f - \hf\).
\begin{remark}[Higher order Sobolev norms]
Let us assume that \(\Support{\q} \cup \Support{\cq}\) is contained in a compact set \(\USsp\) for all \(\hf\) considered.
Note that for \(d \in [D]\),
\begin{equation*}\begin{split}
\int_\USsp |h(\e)|^d \de \leq (\UEspVol)^{\frac{d}{d - D}} \left(\int_\USsp |h(\e)|^D \de\right)^{d/D}
\end{split}\end{equation*}
by \Holder's inequality, where we defined \(\UEspVol := \int_\USsp 1 \de\),
hence we have \(\LdWith{\USsp}{\cdot} \leq (\UEspVol)^{\frac{1}{d - D}} \LDWith{\USsp}{\cdot}\).
By applying the relation to each term in the definition of \(\Sobolevd{\cdot}\), we obtain
\begin{equation*}\begin{split}
\Sobolevd{f}^d \leq (\UEspVol)^{\frac{d}{d - D}} \SobolevD{f}^d
\end{split}\end{equation*}
Thus we obtain
\begin{equation*}\begin{split}
\ApproxDensityDifferenceBoundRemainder &= \ApproxDensityDifferenceBoundRemainderDef \\
&\leq \sum_{d=2}^D \comb{D}{d} (\UEspVol)^{\frac{d}{d - D}} \ApproxDensityDifferenceJacobianPieceCoeff \sum_{j=1}^D \SobolevD{\f_j - \hf_j}^d \\
&\leq \Order\left(\sum_{j=1}^D \SobolevD{\f_j - \hf_j}^2\right) \quad (\hf \to \f).
\end{split}\end{equation*}
By also replacing \(\sum_{j=1}^D\SobolevOne{\f_j - \hf_j}\) with \(\sum_{j=1}^D\SobolevD{\f_j - \hf_j}\) in Theorem~\ref{thm:generalization-error},
we can see more clearly that \(\ApproxDensityDifferenceBoundRemainder\) is a higher order term of \(\sum_{j=1}^D\SobolevD{\f_j - \hf_j}\).
\label{theory:remark:higher-order-sobolev}
\end{remark}
\section{Details and Proofs of Theorem~\ref{paper:thm:1} \label{paper:sec:appendix:theory-1}}
\label{sec:orga8779b8}
Here, we provide the proof of Theorem~\ref{paper:thm:1}.
We reuse the notation and terminology from Section~\ref{paper:sec:appendix:theory-2} of this \Supplementary.
We prove the uniformly minimum variance property of the proposed risk estimator under the ideal situation of \(\hf = \f\).
\begin{theorem}[Known causal mechanism case]
Assume \(\hf = \f\). Then, for all \(\g \in \G\), we have that \(\cbarhR(\g)\) is the uniformly minimum variance unbiased estimator of \(\R(\g)\).
As a special case, it has a smaller variance than the ordinary empirical risk estimator: \(\allq, \allg, \Var(\cbarhR(\g)) \leq \Var(\hR(\g))\).
\label{thm:known-causal-mechanism-case}
\end{theorem}
\begin{proof}
The proof is a result of the following two facts.
When \(\cq \in \Qsp\), the estimator \(\cbarhR(\g)\) becomes the generalized U-statistic of the statistical functional Eq.~\eqref{eq:v-estimator-asymptotic-target}.
Furthermore, when \(\hf = \f\), Eq.~\eqref{eq:v-estimator-asymptotic-target} coincides with \(\R(\g)\) because the approximation error is zero.
Since we assume \(\hf = \f\) we have \(\cq = \q \in \Qsp\) and hence both of the statements above hold.
Therefore, by Lemma~\ref{lem:generalized-U-statistics-is-UMVUE}, the first assertion of the theorem follows.
The last assertion of the theorem follows as a special case as \(\hR(\g)\) is an unbiased estimator of \(\R(\g)\) for \(\q \in \Qsp\).

From here, we confirm the above statements by calculation.
We first show that \(\cbarhR(\g)\) is the generalized U-statistic.
To see this, observe that the statistical functional Eq.~\eqref{eq:v-estimator-asymptotic-target} allows the following expression given \(\cq \in \Qsp\):
\begin{equation*}\begin{split}
&\int \ke(\e_1, \ldots, \e_D) \cq(\e_1) \cdots \cq(\e_D) \de_1 \cdots \de_D \\
&= \int \DSymmetricMean \l(\g, \hf(\e_{\pi(1)}^{(1)}, \ldots, \e_{\pi(D)}^{(D)})) \cq(\e_1) \cdots \cq(\e_D) \de_1 \cdots \de_D\\
&= \int \DSymmetricMean \l(\g, \hf(\e_{\pi(1)}^{(1)}, \ldots, \e_{\pi(D)}^{(D)})) \prodd \cqd(\e_1^{(d)}) \cdots \prodd \cqd(\e_D^{(d)}) \de_1 \cdots \de_D \\
&= \int \DSymmetricMean \l(\g, \hf(\e_1^{(1)}, \ldots, \e_1^{(D)})) \prodd \cqd(\e_1) \de_1 \\
&= \int \l(\g, \hf(\e^{(1)}, \ldots, \e^{(D)})) \cq_1(\e^{(1)}) \cdots \cq_D(\e^{(D)}) \de^{(1)} \cdots \de^{(D)}. \\
\end{split}\end{equation*}
This is a regular statistical functional of degrees \((1, \ldots, 1)\) with the kernel \(\l(\g, \hf(\cdot, \ldots, \cdot))\).
On the other hand, we have
\begin{equation*}\begin{split}
\cbarhR(\g) &= \frac{1}{n^D} \sum_{(i_1, \ldots, i_D) \in [n]^D} \ke(\S_{i_1}, \ldots, \S_{i_D}) = \frac{1}{n^D} \sum_{(i_1, \ldots, i_D) \in [n]^D} \l(\g, \hf(\S_{i_1}^{(1)}, \ldots, \S_{i_D}^{(D)}))
\end{split}\end{equation*}
because the summations run through all combinations with replacement.
This combined with the fact that \(\{\S_i^{(d)}\}_{i, d}\) are jointly independent when \(\cq \in \Qsp\) yields that
\(\cbarhR(\g)\) is the generalized U-statistic for Eq.~\eqref{eq:v-estimator-asymptotic-target}.

Now we show that Eq.~\eqref{eq:v-estimator-asymptotic-target} coincides \(\R(\g)\).
Given \(\hf = \f\), we have
\begin{equation*}\begin{split}
\R(\g) &= \int \q(\e) \l(\g, \f(\e)) \de \\
&= \int \q(\e) \l(\g, \hf(\e)) \de \quad (\text{By }\f = \hf.)\\
&= \int \q_1(\e^{(1)}) \cdots \q_D(\e^{(D)}) \l(\g, \hf(\e^{(1)}, \ldots, \e^{(D)})) \de^{(1)} \cdots \de^{(D)} \quad (\text{by }\q \in \Qsp)\\
&= \int \cq_1(\e^{(1)}) \cdots \cq_D(\e^{(D)}) \l(\g, \hf(\e^{(1)}, \ldots, \e^{(D)})) \de^{(1)} \cdots \de^{(D)} \quad (\text{by }\q = \cq)\\
&= \int \ke(\e_1, \ldots, \e_D) \cq(\e_1) \cdots \cq(\e_D) \de_1 \cdots \de_D. \quad (\because \text{ symmetry}) \\
\end{split}\end{equation*}
\end{proof}

The following well-known lemma states that a generalized U-statistic is a uniformly minimum variance unbiased estimator.
\begin{lemma}[Uniformly minimum variance property of a generalized U-statistic]
Let \(\th: \Qsp \to \Re\) be a regular statistical functional with kernel \(\psi: \Re^{k_1} \times \cdots \times \Re^{k_L} \to \Re\) \citep{ClemenconScalingup2016}, i.e.,
\begin{equation*}\begin{split}
\th(\q) = \int \psi((x_1^{(1)}, \ldots, x_{k_1}^{(1)}), \ldots, (x_1^{(L)}, \ldots, x_{k_L}^{(L)})) \prod_{j=1}^{k_1}\q_1(x_j^{(1)}) \d x_j^{(1)}  \cdots \prod_{j=1}^{k_L}\q_L(x_j^{(L)})\d x_j^{(L)}.
\end{split}\end{equation*}
Given samples \(\{x_i^{(l)}\}_{i=1}^{n_l} \iid \q_l (n_l \geq k_l \text{ and } l = 1, \ldots, L)\), let \(\GUn \psi\) be the corresponding generalized U-statistic
\begin{equation*}\begin{split}
\GUn \psi := \frac{1}{\prod_l \comb{n_l}{k_l}} \sum_{} \psi\left(\left(x_{i_1^{(1)}}^{(1)}, \ldots, x_{i_{k_1}^{(1)}}^{(1)}\right), \ldots, \left(x_{i_1^{(L)}}^{(L)}, \ldots, x_{i_{k_L}^{(L)}}^{(L)}\right)\right).
\end{split}\end{equation*}
where \(\sum\) denotes that the indices run through all possible combinations (without replacement) of the indices.
Then, \(\GUn\psi\) is the uniformly minimum variance unbiased estimator of \(\th\) on \(\Qsp\).
\label{lem:generalized-U-statistics-is-UMVUE}
\end{lemma}

\begin{proof}
The assertion can be proved in a parallel manner as the proof of \citep[Section 1.1, Lemma~B]{LeeUStatistics1990}
\end{proof}
\begin{remark}[Relation to the UMVUE property of \(\hR(\g)\)]
The result in Theorem~\ref{thm:known-causal-mechanism-case} is not contradictory to the fact that the sample average \(\hR(\g)\) is a U-statistic of degree-\(1\) and hence the minimum variance among all unbiased estimator of \(\R(\g)\) on \(\PDsp\), where \(\PDsp\) is a set of distributions containing all absolutely continuous distributions \cite{LeeUStatistics1990}.
Specifically, \(\cbarhR(\g)\) is not generally an unbiased estimator of \(\R(\g)\) on \(\PDsp \setminus \Qsp\), even if \(\hf = \f\).
While \(\cbarhR(\g)\) satisfies the \(D\)-sample symmetry condition, the same does not hold for \(\hR(\g)\).
By restricting the attention to \(\Qsp\), the estimator \(\cbarhR(\g)\) achieves a smaller variance than \(\hR(\g)\).
\end{remark}
\section{Further Comparison with Related Work \label{paper:sec:appendix:further-related-work}}
\label{sec:org5e55acf}
Here, we provide an additional detailed comparison with the related work
to complement Section~\ref{paper:sec:related-work} of the main text.
\subsection{Comparison with \citet{MagliacaneDomain2018}}
\label{sec:org2383df3}
\citet{MagliacaneDomain2018} considered domain adaptation among different interventional states by using SCMs.
Their problem setting and ours do not strictly include each other (the two settings are somewhat complementary),
and their assumption may be more suitable for application fields with interventional experiments such as genomics,
while ours may be more suited for fields with observational data such as health record analysis or economics.
At the methodological level, \citet{MagliacaneDomain2018} takes a variable selection approach to find a subset so that the conditional distribution is invariant,
whereas our paper takes a data augmentation approach via the estimation of the SEMs (in the reduced form).

The essential assumptions of \citet{MagliacaneDomain2018} are the existence of a separating set (with small ``incomplete information bias'') and the identifiability of such a set (yielded from Proposition 1, Assumption 1, and Assumption 2 (iii) in \citet{MagliacaneDomain2018}).
A particularly plausible application conforming to the assumptions is, for example (but not limited to), genomics experiments.
Part of the reason is that Assumption 2 (ii) and (iii) are likely to hold for well-targeted experiments \cite{MagliacaneDomain2018}.
The following is a detailed comparison.

\paragraph{(1) Modeling assumption and problem setup.}
\label{sec:org66ec14f}
The two problem settings do not strictly include one another, and they are of complementing relations where ours corresponds to the intervention-free case and \citet{MagliacaneDomain2018} corresponds to the intervention case.
If we try to express the problem setting of \citet{MagliacaneDomain2018} within our formulation, we would be expressing the interventions as alterations to the SEMs.
We assume that such alterations do not occur in our setting since our focus is on observational data; therefore, the problem formulation of \citet{MagliacaneDomain2018} is not a subset of ours.
On the other hand, if we try to express our problem setting within the formulation of \citet{MagliacaneDomain2018}, our problem setup would only have \(C_1\) as the context variable, and \(C_1\) would be a parent of all observed variables, e.g., \(C_1\) switches the distribution of \(\S\) by switching different quantile functions to perform inverse transform.
This potentially allows the existence of the effect \(C_1 \to Y\) and diverges from Assumption 2 (iii) in \citet{MagliacaneDomain2018}.
Also, even if such an edge does not exist, it is acceptable that there are no separating sets (in the extreme case) if \(Y\) is a parent of all \(X_i\)'s.
In this case, conditioning on any of the \(X_i\)'s would result in making \(C_1\) and \(Y\) dependent.
From this consideration, our problem setting is not a subset of that of \citet{MagliacaneDomain2018}, either.

\paragraph{(2) Plausible applications.}
\label{sec:org46bb7a9}
The problem setup of \citet{MagliacaneDomain2018} is suitable especially for applications in which various experiments are conducted such as genomics \cite{MagliacaneDomain2018},
whereas our problem setting may be more suitable for some fields with observational data such as health record analysis or economics.

\paragraph{(3) Methodology.}
\label{sec:org766e3cf}
Our proposed method actually estimates the SEMs (though in the reduced-form) and exploits the estimated SEMs in the domain adaptation algorithm.
In fact, directly using the estimated SEMs as a tool to realize domain adaptation can be seen as the first attempt to fully leverage the structural causal models in the DA algorithm.
On the other hand, \citet{MagliacaneDomain2018} approaches the problem of domain adaptation via variable selection to find a subset so that the conditional distribution is invariant.
\subsection{Comparison with \citet{GongCausal2018} \label{paper:sec:appendix:CG-DAN-comparison}}
\label{sec:orgb879c8f}
In the present paper, we assumed an invariance of structural equations between domains.
Here, we clarify the difference from a related but different assumption considered by Causal Generative Domain Adaptation Network (CG-DAN; \citealp{GongCausal2018}).
\paragraph{(1) Problem setup.}
\label{sec:orgac6ed21}
\citet{GongCausal2018} presumes the \emph{anticausal} scenario (i.e., \(Y\) is the cause of \(X\)) and that \(X\) given \(Y\) follows a structural equation model, whereas our paper considers more general SEMs of \(X\) and \(Y\).
\paragraph{(2) Theoretical justification.}
\label{sec:org4d02e32}
The approach of \citet{GongCausal2018} does not have a theoretical guarantee in terms of the identifiability of \(\f\), i.e., there has been no known theoretical condition under which the learned generator is applicable across different domains.
On the other hand, our method enjoys a strong theoretical justification of nonlinear ICA including the identifiability of \(\f\) under known theoretical conditions.
\paragraph{(3) Methodology.}
\label{sec:org3f29ff0}
The method of \citet{GongCausal2018} estimates the GCM of \(X\) given \(Y\) using source domain data and uses it to design a generator neural network.
On the other hand, we more directly exploit the estimated reduced-form SEM in the method.
\subsection{Comparison with \citet{ArjovskyInvariant2020}}
\label{sec:org6bada49}
\citet{ArjovskyInvariant2020} proposed \emph{invariant risk minimization} (IRM) for the \emph{out-of-distribution (OOD) generalization} problem.
The IRM approach tries to learn a feature extractor that makes the optimal predictor invariant across domains, and its theoretical validity is argued based on SCMs.
Here, we compare it with the present work in terms of the problem setup, theoretical justification, and the methodology.
\paragraph{(1) Basic assumption and problem setup.}
\label{sec:org2c63694}
The OOD generalization problem tackled in \citet{ArjovskyInvariant2020} assumes no access to the target domain data.
In this respect, the problem is different and intrinsically more difficult than the one considered in this paper, where a small labeled sample from the target domain is assumed to be available.
In order to solve the OOD generalization problem, in a nutshell, \citet{ArjovskyInvariant2020} essentially assumes the existence of a feature extractor that \emph{elicits an invariant predictor}, i.e., one that makes the optimal predictors of the different domains to be identical after the feature transformation. This can be seen as a variant of the representation learning approach for domain adaptation where we assume there exists \(\mathcal{T}\) such that \(p(Y|\mathcal{T}(X))\) is invariant across domains. Indeed, for example, when the loss function is the cross-entropy, the condition corresponds to the invariance of \(P(Y|\mathcal{T}(X))\) across domains \cite{ArjovskyInvariant2020}. More technically, in addition, \citep[Definition~7(ii)]{ArjovskyInvariant2020} requires the condition \(\E_1[Y|\mathrm{Pa}(Y)] = \E_2[Y|\mathrm{Pa}(Y)]\), which can be violated when the latent factors corresponding to \(Y\) have different distributions across domains.
On the other hand, our assumption can be seen as the existence of a feature extractor that can simultaneously estimate the independent components in all domains, which does not necessarily imply the existence of a common feature transformer that induces a unique optimal predictor.
\paragraph{(2) Theoretical justification.}
\label{sec:orga1d512c}
\citet{ArjovskyInvariant2020} formulated a condition under which the IRM principle leads to an appropriate predictor for OOD generalization, but only under a certain linearity assumption which is essentially a relaxation of linear SEMs. Furthermore, in the theoretical guarantee, the feature extractor is restricted to be linear. In addition, \citet{ArjovskyInvariant2020} only provides the population-level analysis that the solution of the IRM objective formulated using the underlying distributions enjoys OOD generalization, and it does not discuss the condition under which the ideal feature extractor can be properly estimated by the empirical IRM.
The requirement for the strong assumption of linearity likely stems from the intrinsic difficulty of the OOD problem in \citet{ArjovskyInvariant2020}, namely, its formulation does not assume specific types of interventions.
On the other hand, our method enjoys a stronger theoretical guarantee of an excess risk bound without such parametric assumptions on the models or the data generating process, by focusing on the case that the causal mechanisms are indifferent across the domains.
\paragraph{(3) Methodology.}
\label{sec:orgc566f50}
The methodology of IRM estimates a single predictor that generalizes well to all domains by finding a feature extractor that makes the predictor optimal in all domains.
The approach shares the same spirit as the representation learning approaches to domain adaptation, which try to find a feature extractor that induces invariant conditional distributions, such as transfer component analysis \cite{PanDomain2011}.
On the other hand, our method estimates the SEMs (in the reduced-form) and exploits it to make the training on the few target domain data more efficient through data augmentation.

\end{appendices}
\end{document}